\documentclass[twoside,11pt]{article}

\usepackage[preprint]{jmlr2e}

\usepackage{subfig}
\usepackage[usenames]{color}
\usepackage{enumerate}
\usepackage{url}
\usepackage{amsfonts,mathrsfs}
\usepackage{amssymb,amsmath}
\usepackage{verbatim}
\usepackage{acronym}
\usepackage{mathtools}
\usepackage{thmtools,thm-restate}
\usepackage{graphicx}
\usepackage{algorithm}
\usepackage[noend]{algpseudocode}
\def\fskip#1{}

\def\E{\mathbb{E}}

\def\H{\mathcal{H}}
\def\G{\mathcal{G}}
\def\V{\mathcal{V}}

\def\CP{\mathcal{P}}

\def\argmax{\mathop{\rm argmax}}

\newtheorem{assumption}{Assumption}

\usepackage{lastpage}

\ShortHeadings{}{}
\firstpageno{1}

\usepackage{times}

\begin{document}

\title{\vspace{-0.3cm} Fully Online Decentralized Learning in Stochastic Games with Unknown Independent Chains}

\author{%
	\name{S. Rasoul Etesami} \email{etesami1@illinois.edu}\\
	\addr Department of Industrial and Systems Engineering, Coordinated Science Laboratory\\ University of Illinois Urbana-Champaign, Urbana, IL, USA 61801}


\maketitle

\begin{abstract}%
We consider stochastic games with independent controlled chains and unknown transition kernels, in which players observe only their own local states and realized payoffs. We develop a fully online decentralized and uncoordinated mirror-descent algorithm that operates in the dual space of occupancy measures for approximating stationary Nash equilibrium (NE) policies. The algorithm uses a single transition/reward sample at every primitive time step, allows players to update using only local information, and requires neither coverage of the joint state space nor synchronized episodes. Under standard uniform-ergodicity and finite-coverage assumptions, we show that, with high probability, the time-averaged fixed-comparator regret decays at the canonical $\widetilde{\mathcal{O}}(T^{-1/2})$ rate, up to logarithmic factors and polynomial dependence on the game parameters. In particular, the complexity depends on the cover times of the individual local state spaces rather than on that of the product state space, thereby avoiding exponential dependence on the number of players and the sizes of the joint state and action spaces. We further show that the resulting finite-time regret bound provides an approximate coarse-correlated-equilibrium guarantee. Such a guarantee is natural for arbitrary reward functions, since, in this setting, computing a stationary $\epsilon$-Nash equilibrium ($\epsilon$-NE) is known to be PPAD-hard. Under an additional global variational-stability condition, we show that the same fully online algorithm converges asymptotically in the last iterate to a stationary $\epsilon$-NE. Our results thus provide the first fully online and scalable learning framework for stochastic games with unknown independent chains. Our algorithmic development can also be viewed as a primal-dual framework for Markov games that fully exploits the independence and local structure of the players' controlled transition chains and may be of independent interest.
\end{abstract}

\section{Introduction}

Learning in multi-agent systems becomes particularly challenging when strategic interactions, dynamic stochastic state evolution, and limited information are present simultaneously \citep{zhang2021multi}. Unlike in a single-agent Markov decision process \citep{altman2021constrained}, where the environment is typically viewed as exogenous, each agent in a stochastic game faces an environment that is partly shaped by the evolving decisions of other strategic agents \citep{shapley1953stochastic,basar1999dynamic}. Consequently, even when the underlying transition dynamics are stationary, the effective optimization problem faced by each player changes as the other players adapt. This coupling raises a fundamental question in decentralized learning: under what structural conditions can self-interested agents, using only locally available observations, learn meaningful equilibrium behavior without coordination or knowledge of the underlying model?

For general-sum games \citep{nash1950equilibrium}, one cannot expect an affirmative answer without additional structure. Computing a \emph{Nash equilibrium} (NE) is PPAD-hard even for finite normal-form games \citep{daskalakis2009complexity}, and stochastic games contain such static games as special cases. Much of the algorithmic literature has therefore focused on structured settings in which equilibrium learning becomes tractable \citep{mertikopoulos2019learning}. Examples include two-player zero-sum stochastic games, where fictitious-play or independent-learning procedures can achieve asymptotic or finite-sample convergence guarantees \citep{daskalakis2021independent,sayin2022fictitious,zhao2021provably}, and \emph{Markov potential games} (MPGs), where the potential structure can be exploited to establish global convergence guarantees for policy-gradient-type methods \citep{fox2022independent,etesami2026markovpotential,maheshwari2025independent,zhang2021gradient,leonardos2021global}. These results demonstrate how suitable structure can fundamentally alter the computational and statistical difficulty of multi-agent learning. However, their structural assumptions may exclude systems in which agents have distinct objectives while their  state processes evolve largely independently.

This paper studies a different source of tractability: \emph{independence in the controlled state dynamics}. We consider $n$-player stochastic games in which each player controls its own finite-state Markov chain. A player's local transition kernel depends only on its own state and action, while strategic interactions occur through the reward functions, which may depend on the states and actions of all players. This model is natural in systems where individual agents have local physical dynamics but interact through a common environment, prices, congestion, or shared resources. Examples include wireless communication systems \citep{altman2007constrained,narayanan2017large}, decentralized energy management \citep{etesami2018stochastic}, and decentralized robotic navigation and coverage \citep{marden2009cooperative}. Importantly, independence of the transition processes does not make the game strategically separable: a player's payoff can still depend arbitrarily on the behavior of all other players. Our information model is deliberately restrictive. Each player observes only its own local state, chosen action, realized payoff, and subsequent local state; it neither observes the states or actions of the other players nor knows its own transition kernel. Thus, each player must simultaneously learn its local dynamics and adapt strategically to the other players from a single continuing trajectory. Moreover, we seek algorithms whose complexity is governed by the sizes and mixing properties of the \emph{local} controlled chains rather than those of the product state--action space. This distinction is important because the joint state and action spaces grow exponentially with the number of players, even when each individual controlled chain is small.

The independence structure permits a useful reformulation of the game in terms of stationary occupancy measures. Under stationary policies, independence of the local chains implies that the long-run joint state--action distribution factors into the product of local stationary occupancies \citep{etesami2022learning}. Consequently, each player's average reward is linear in its own occupancy measure when the occupancies of the other players are fixed. This converts the original dynamic interaction into a continuous-action static virtual game over local occupancy polytopes. When the transition kernels are unknown, however, these polytopes are themselves unknown and must be learned from data. This creates a coupled learning problem in which players must perform online optimization while simultaneously estimating the feasible regions over which the optimization is carried out. 

A further difficulty is that a genuinely online implementation cannot rely on the episodic simplifications commonly used to separate estimation from optimization. If policies are frozen for sufficiently long episodes, transition and payoff-gradient estimates can be formed from approximately stationary data before the next policy update. Along a continuing online trajectory, by contrast, policies change while data are being generated. The state distribution therefore tracks a moving stationary distribution, samples are temporally dependent, and confidence regions for the unknown kernels evolve asynchronously across players. Establishing regret guarantees in this setting requires controlling all of these effects without introducing synchronization among the players. 

The objective of this paper is to show that the local independence structure is strong enough to overcome these difficulties. We develop a fully online occupancy-based mirror-descent method in which each primitive interaction produces one transition/reward observation and one learning update. Each player maintains a confidence set for its unknown local transition kernel, which is updated only when a local state--action count reaches a dyadic threshold. Between such updates, the feasible occupancy region remains unchanged. This lazy confidence-update mechanism is useful both computationally and analytically: each player updates its model independently, its feasible region changes only logarithmically many times, and no global episode boundaries or synchronization signals are required.

For arbitrary reward functions, unconditional convergence to a NE cannot in general be expected, given the computational hardness inherited from general-sum static games. We therefore first analyze the algorithm through fixed-comparator regret and establish a high-probability $\widetilde{\mathcal{O}}(T^{-1/2})$ average regret bound against every fixed feasible stationary deviation. This directly yields an approximate \emph{coarse-correlated equilibrium} guarantee for the empirical distribution of play \citep{cesa2006prediction}. Crucially, the resulting complexity depends only on local state and action dimensions, mixing parameters, and cover times, rather than on the cover time or cardinality of the exponentially larger joint state--action space. We then identify conditions under which the stronger conclusion of NE convergence can be recovered. Under an additional variational-stability condition \citep{mertikopoulos2019learning} and diminishing step sizes, we show that the same decentralized fully online procedure converges in the last iterate to the unique stationary $\epsilon$-NE with arbitrarily high probability.

\subsection{Related Work}

Our work lies at the intersection of learning in stochastic games, online learning in MDPs, and decentralized primal-dual methods for stochastic games. For general stochastic games, much of the finite-sample literature has focused on settings with additional game-theoretic structure. Independent policy-gradient and Q-learning methods have been shown to converge in two-player zero-sum stochastic games under decentralized learning protocols \citep{daskalakis2021independent,sayin2021decentralized}. Another important line of work concerns MPGs, where policy-gradient and gradient-play methods exploit the potential structure to obtain global convergence guarantees for stationary NE policies \citep{zhang2022global,zhang2021gradient,jordan2024independent,leonardos2021global,ding2022independent,sun2023provably}. These results leverage restrictions on payoff interactions to make equilibrium computation tractable, whereas our setting allows general reward coupling and derives scalability from the independence of the players' controlled transition processes. More recently, several works have combined independent transition dynamics with the Markov potential structure. \citep{nakhleh2025fictitious} studied fictitious play in product Markov games with Kullback--Leibler (KL) control costs and established asymptotic convergence to a NE for the subclass of product MPGs. \citep{jordan2026independent} studied MPGs with independent chains under a more restrictive information structure in which players only partially observe their local states, and established communication-free learning of approximate NE with quasi-polynomial sample and computational complexity.

A related recent line of work studies decentralized learning in \emph{Markov $\alpha$-potential} games \citep{guo2025markov,das2024learning,etesami2026markovpotential}, which generalize MPGs by allowing violations of the Markov potential property up to an additive error $\alpha$. \citep{guo2025markov} derived Nash-regret guarantees for projected gradient-ascent and sequential maximum-improvement schemes under exact Q-function oracle access. \citep{das2024learning} introduced constrained $\alpha$-Markov potential games, which further accommodate coupled constraints among players, and developed decentralized best-response-based algorithms with finite-time guarantees for approximate NE. A fully online decentralized learning framework with finite-time guarantees has only recently been developed for Markov $\alpha$-potential games, where KL-projected natural policy-gradient algorithms achieve high-probability Nash-regret guarantees under both episodic and fully online learning \citep{etesami2026markovpotential}. Like the algorithm developed here, its fully online method operates along a continuing trajectory using one realized sample per interaction. The two frameworks, however, exploit different sources of tractability: the former relies on an approximate Markov-potential structure in discounted Markov games, whereas our work imposes no potential structure on the rewards and instead exploits independent controlled chains in an average-reward setting.

Our analysis also builds on ideas from online learning in MDPs. Occupancy measures provide a convex representation of stationary policies and have been widely used to develop regret-minimization algorithms for adversarial MDPs \citep{cardoso2019large}. In particular, online mirror-descent methods over occupancy polytopes yield sublinear regret when losses vary across episodes, while confidence-set constructions extend such methods to unknown transition kernels \citep{rosenberg2019online}. Related work has also considered bandit feedback with unknown dynamics in online MDPs \citep{jin2020learning}. These single-agent techniques motivate our use of occupancy optimization and confidence regions, but do not directly apply to our setting: each player's effective reward vector is generated endogenously by the evolving policies and state distributions of the other players, while all players learn simultaneously along their own continuing trajectories.

More directly, stochastic games with independent state processes have been studied as a structured class of decentralized dynamic games. \citep{altman2007constrained} established equilibrium existence for cost-coupled stochastic games with independent state processes, and subsequent work characterized stationary equilibria through mathematical programming formulations \citep{singh2014characterization}. For learning problems, \citep{qiu2021provably} showed that, in two-player zero-sum stochastic games with independent chains, fictitious-play policy-optimization algorithms can achieve $\widetilde{\mathcal{O}}(\sqrt{T})$ regret, assuming that each player has knowledge of the opponent's past strategies. \citep{etesami2022learning} introduced an occupancy-based dual formulation and developed decentralized dual-averaging procedures for $n$-player stochastic games with independent chains. That work demonstrated that independence can be leveraged as a source of computational tractability and established finite-time convergence guarantees in terms of the averaged \emph{Nikaido--Isoda gap} \citep{nikaido1955note}. Its main algorithms, however, assume known local transition kernels and operate in an episodic fashion, thereby requiring synchronization among players. Unknown transition dynamics were subsequently addressed by \citep{qin2024learning}, who combined the occupancy formulation with confidence sets and decentralized online mirror descent. Their algorithm learns the local transition models from observed trajectories and, under a variational-stability assumption, converges asymptotically to a stationary $\epsilon$-NE. The learning protocol, however, remains episodic: policies are held fixed while sufficiently many local observations are collected, and an episode-ending signaling mechanism coordinates policy updates across players. Extending this framework to a fully online setting presents significant additional challenges, as continuous policy updates eliminate the stationary sampling environment provided by episodic learning and synchronization.

The present paper removes this episodic architecture through delayed conditioning, coupling with frozen-policy dynamics, sensitivity bounds for policy and kernel movement, tracking of time-varying stationary occupancies, and careful treatment of asynchronous confidence-set updates. Players learn from every primitive interaction, confidence sets evolve asynchronously using only local information, and no episode boundaries, dedicated covering phases, or synchronization signals are required. Beyond this algorithmic distinction, we establish a finite-time, high-probability fixed-comparator regret bound for arbitrary reward functions and distinguish the resulting coarse-correlated-equilibrium guarantee from the stronger last-iterate $\epsilon$-NE convergence guarantee, which requires an additional variational-stability assumption.

\subsection{Contributions}

We study stochastic games with independent controlled chains and unknown transition kernels and develop a fully decentralized online algorithm for learning stationary equilibrium policies. Our main contributions are as follows.
\begin{itemize}
\item We develop a fully online, decentralized, and uncoordinated algorithm that uses one sample per time step. At each primitive time $t$, every player observes one local transition and realized payoff, updates its local transition counts, forms an importance-weighted payoff-gradient estimate, and takes one mirror-descent step. The algorithm requires no covering episodes, burn-in phases, or synchronization signals. Confidence sets are updated asynchronously only when a local state--action count reaches a dyadic threshold, so the feasible occupancy region remains fixed between updates and changes only logarithmically many times over a finite horizon.

\item For arbitrary reward functions, we establish a high-probability fixed-comparator regret bound of order $\widetilde{\mathcal O}(T^{-1/2})$ after a local-cover transient period. The bound depends only on local state and action dimensions, mixing parameters, and cover times, rather than on those of the exponentially larger joint state--action space. Consequently, the empirical distribution of play satisfies an approximate coarse-correlated-equilibrium guarantee.

\item We further identify a regime in which the stronger NE conclusion can be recovered. Under an additional variational-stability assumption and diminishing step sizes, we prove that the same fully online decentralized algorithm converges asymptotically in the last iterate to the unique stationary $\epsilon$-NE with arbitrarily high probability. Thus, the finite-time no-regret guarantee requires no structural assumptions on the reward functions, while last-iterate NE convergence follows under stability of the induced occupancy game.
\end{itemize}

\subsection{Organization}

The rest of the paper is organized as follows. In Section~\ref{sec:formulation}, we formulate stochastic games with independent controlled chains and unknown transition kernels and introduce the stationary equilibrium concepts used throughout the paper. Section~\ref{sec:dual} develops the occupancy-measure representation, the associated virtual game, and the local coverage quantities needed for the analysis. In Section~\ref{sec:alg-design}, we introduce the fully online decentralized mirror-descent algorithm together with its asynchronous confidence-set updates. Section~\ref{sec:finite_analysis} establishes the finite-time, high-probability fixed-comparator regret bound and derives the resulting coarse-correlated-equilibrium guarantee. Section~\ref{sec:asymptotic} studies the diminishing-step-size regime and establishes last-iterate convergence to an $\epsilon$-NE under variational stability. Section~\ref{sec:applications} discusses applications and implications of the framework. Finally, Section~\ref{sec:conclusion} concludes the paper and outlines directions for future research. Omitted proofs and additional technical details are provided in the appendices.

\paragraph{Notation.}
For a positive integer $n$, we write $[n]:=\{1,\ldots,m\}$. For a finite set $\mathcal X$, $|\mathcal X|$ denotes its cardinality. Vectors are indexed by their corresponding states, actions, or state--action tuples, and $\langle x,y\rangle$ denotes the standard Euclidean inner product. We use $\|x\|_1$, $\|x\|_2$, and $\|x\|_\infty$ for the $\ell_1$, Euclidean, and supremum norms, respectively, and ${\bf 1}\{\cdot\}$ for the indicator function. For a profile $x=(x_1,\ldots,x_n)$, we write $x_{-i}:=(x_j)_{j\neq i}$ for the profile excluding player $i$. Lowercase letters such as $s_i$ and $a_i$ denote local states and actions, while $s=(s_1,\ldots,s_n)$ and $a=(a_1,\ldots,a_n)$ denote joint profiles; uppercase letters denote the corresponding random variables. We use $\mathbb P$ and $\mathbb E$ for probability and expectation, respectively, and $\mathcal F_t$ for the filtration in the probabilistic analysis. For occupancy measures, $\nu_i$, $\rho_i$, and $q_i$ denote state, state--action, and state--action--next-state occupancies, respectively. Throughout, unhatted quantities $(q_i^t,\rho_i^t,\nu_i^t)$ denote the  occupancies induced by the current policy $\pi_i^t$ under the true kernel $P_i$, whereas hatted quantities $(\hat q_i^t,\hat\rho_i^t,\hat\nu_i^t)$ denote the confidence-feasible quantities maintained by the algorithm. We write $x\lesssim y$ if $x\le \alpha y$ for some universal constant $\alpha>0$. Throughout, we use $c>0$ to denote an adjustable universal positive constant. Finally, $\mathcal{O}(\cdot)$ and $\widetilde{\mathcal{O}}(\cdot)$ hide universal constants and, in the latter case, logarithmic factors.

\vspace{-0.3cm}
\section{Problem Formulation}\label{sec:formulation}

We consider an $n$-player infinite-horizon, time-average stochastic game with independent state transitions, described by the tuple
$\mathcal{G}:=\big([n], \{\mathcal{S}_i\}_{i=1}^n, \{\mathcal{A}_i\}_{i=1}^n, \{r_i\}_{i=1}^n, \{P_i\}_{i=1}^n\big)$, as follows.
\begin{itemize}
	\item $\mathcal S_i$ is the finite state space of player $i$, with elements $s_i\in \mathcal S_i$. We denote the joint state space of all players by $\mathcal S = \mathcal S_1\times \cdots\times \mathcal S_n$, with elements $s\in \mathcal S$, where $s=(s_1,\ldots,s_n)$. 
	\item $\mathcal A_i$ is the finite action space of player $i$, with elements $a_i\in \mathcal A_i$. We denote the joint action space of all players by $\mathcal A = \mathcal A_1\times \cdots \times \mathcal A_n$, with elements $a\in\mathcal A$, where $a=(a_1,\ldots,a_n)$. 
	\item $r_i:\mathcal S\times\mathcal A\to[0,1]$ is the reward function of player $i$, where $r_i(s,a)$ is the immediate reward received by player $i$ when the joint state is $s=(s_1,\ldots,s_n)$ and the players take the joint action $a=(a_1,\ldots,a_n)$.
	\item $P_i$ is the transition probability kernel of player $i$, where $P_i(s_i'\mid s_i,a_i)$ is the probability that the state of player $i$ transitions from $s_i$ to $s_i'$ when the player chooses action $a_i$. Crucially, $P_i$ is \emph{unknown} to player $i$, and player $i$'s next-state distribution depends only on $(s_i,a_i)$. The precise factorization of the joint transition kernel is stated in the following assumption.
\end{itemize}

\begin{assumption}\label{ass:indep}
We assume that the joint transition probability kernel $P(s'\mid s,a)$ factors into independent components as $P(s'\mid s,a)=\prod_{i=1}^n P_i(s_i'\mid s_i,a_i)$, where $P_i(s_i'\mid s_i,a_i)$ denotes the (unknown) transition probability of player $i$ from state $s_i$ to state $s_i'$ under action $a_i$.
\end{assumption}

At any time $t$, the information available to player $i$ is given by the history of its realized states, actions, and rewards, i.e.,
$\H_i^t = \{s^{\ell}_i,a^{\ell}_i,r_i(s^{\ell},a^{\ell}) : \ell = 0,1,\ldots,t-1\}\cup\{s^t_i\}$.
Given the information set $\H_i^t$, player $i$ takes an action $a_i^t$ according to her current policy $\pi_i^t(\cdot\mid\H_i^t)$, which is a probability measure over $\mathcal A_i$, and receives the reward $r_i(s^t,a^t)$, which may also depend on the states and actions of the other players. Subsequently, the state of player $i$ transitions from $s_i^t$ to a new state $s_i^{t+1}$ according to the transition probability $P_i(s_i^{t+1}\mid s_i^t,a_i^t)$. A general policy for player $i$ is a sequence of decision rules $\pi_i=\{\pi_i^t:t=0,1,\ldots\}$, where each $\pi_i^t(\cdot\mid\H_i^t)$ is a probability measure over $\mathcal A_i$. However, the use of general policies is often computationally expensive, and in practical applications, players are typically interested in easily implementable \emph{stationary} policies, as defined next, which are also the main focus of this paper.

\begin{definition}
	A policy $\pi_i$ for player $i$ is called stationary if the probability $\pi_i^t(a_i\mid \H^t)$ of choosing action $a_i$ at time $t$ depends only on the current local state $s^t_i = s_i$, and is independent of the time $t$. In the case of the stationary policy, we use $\pi_i(a_i\mid s_i)$ to denote this time-independent probability.
\end{definition}

Given an initial state $s^0$, the objective of each player $i\in[n]$ is to choose a stationary policy $\pi_i$ that maximizes its long-term expected average reward (payoff), given by
\begin{align}\label{eq:aggregate-reward}
	V_i(\pi_i,\pi_{-i})=\limsup_{T\to \infty}\frac{1}{T}\mathbb{E}\Big[\sum_{t=0}^{T}r_i(s^t,a^t)\Big],
\end{align} 
where $\pi_{-i}=(\pi_j, j\neq i)$, and the expectation is with respect to the randomness introduced by players' internal kernels $(P_1,\ldots,P_n)$ and their stationary policies $\pi=(\pi_1,\ldots,\pi_n)$. 

This fully characterizes the $n$-player stochastic game considered here, in which each player $i$ seeks to choose a stationary policy $\pi_i$ to maximize its expected aggregate payoff $V_i(\pi_i,\pi_{-i})$. In the remainder of the paper, we refer to the above reward-coupled stochastic game with independent chains and unknown transition kernels as the original game $\mathcal{G}$. Next, to establish meaningful convergence and learning guarantees, we impose the following assumption throughout this work.
\begin{assumption}\label{ass:ergodic}
For any player $i$ and any stationary policy $\pi_i$ chosen by that player, the induced Markov chain, with transition probabilities
\[
P^{\pi_i}(s_i'\mid s_i):=
\sum_{a_i\in\mathcal A_i}
P_i(s_i'\mid s_i,a_i)\pi_i(a_i\mid s_i),
\]
is ergodic and satisfies a uniform contraction condition: there exists a constant $\kappa\in(0,1]$ such that\footnote{Here, $\Delta(\mathcal S_i)$ denotes the probability simplex over the local state space $\mathcal{S}_i$.}
\begin{align*}
\|(\nu-\nu')P^{\pi_i}\|_1
\le (1-\kappa)\|\nu-\nu'\|_1,
\qquad
\forall i,\ \pi_i,\ \nu,\nu'\in\Delta(\mathcal S_i).
\end{align*}
\end{assumption}

Assumption~\ref{ass:ergodic} is a uniform-ergodicity condition commonly used in finite-time analyses of average-reward online MDPs (see, e.g., \citealp{rivera2019large,neu2010online,rosenberg2019online}). Its role here is quantitative: the proof uses policy-uniform contraction to control the tracking error of the continuously changing local chains and the bias of the delayed one-sample gradient estimator. We do not claim that such a condition is necessary for learning in general MDPs; weaker communicating assumptions may suffice for other objectives and algorithms. Rather, Assumption~\ref{ass:ergodic} delineates the controlled-chain regime covered by our analysis. As shown in the next section, under Assumption~\ref{ass:ergodic}, for any stationary policy profile $\pi$, the limit in \eqref{eq:aggregate-reward} exists and admits the closed-form expression. 

\begin{definition}
For a stationary policy profile $\pi^*=(\pi_1^*,\ldots,\pi_n^*)$, $\pi_i^*$ is called a stationary best response to $\pi_{-i}^*$ if $V_i(\pi_i^*,\pi_{-i}^*)\ge V_i(\pi_i,\pi_{-i}^*)$ for every stationary policy $\pi_i$. Similarly, $\pi_i^*$ is called a stationary $\epsilon$-best response to $\pi_{-i}^*$ if $V_i(\pi_i^*,\pi_{-i}^*)\ge V_i(\pi_i,\pi_{-i}^*)-\epsilon$ for every stationary policy $\pi_i$. The stationary policy profile $\pi^*$ is a stationary Nash equilibrium (NE) if $\pi_i^*$ is a stationary best response to $\pi_{-i}^*$ for every player $i$, and it is a stationary $\epsilon$-NE if $\pi_i^*$ is a stationary $\epsilon$-best response to $\pi_{-i}^*$ for every player $i$. All equilibrium statements in this paper concern stationary policies.
\end{definition}

Our main objective is to understand the equilibrium guarantees of a decentralized and fully online learning rule for this class of stochastic games. We establish two types of guarantees: a finite-time fixed-comparator regret bound for the time-averaged iterates and, under stronger conditions, last-iterate convergence to a stationary $\epsilon$-NE.

\section{Preliminaries on Occupancy Measures, Dual Formulation, and Local Cover Times}\label{sec:dual}

In this section, we provide an alternative dual formulation of the original stochastic game $\G$ based on \emph{occupancy measures} \citep{altman2021constrained}. Intuitively, from player $i$'s perspective, its long-term expected average payoff depends on the proportion of time it spends in each state $s_i$ and takes each action $a_i$, as captured by its occupancy measure. Thus, the policy optimization problem for player $i$ can be reformulated as an optimization problem over occupancy measures, where each player seeks to steer its chain toward high-reward states and actions. An advantage of optimizing over occupancy measures is that, due to the independence of the players' internal chains, the payoff functions admit a simple closed-form that is easier to analyze than the original formulation in terms of policy variables. Moreover, this simplified payoff structure allows us to leverage the rich literature on online learning and primal-dual methods to address the learning problem. We use this dual formulation to develop learning algorithms for finding a stationary $\epsilon$-NE.

\subsection{Occupancy Measures}

For a given single-agent MDP with state space $\mathcal{S}$, action space $\mathcal{A}$, transition probability kernel $P$, and stationary policy $\pi$, one can associate with $P$ and $\pi$ three notions of occupancy measures: $\nu:\mathcal{S}\to[0,1]$, $\rho:\mathcal{S}\times\mathcal{A}\to[0,1]$, and $q:\mathcal{S}\times\mathcal{A}\times\mathcal{S}\to[0,1]$, defined as follows.\footnote{The first type of occupancy measure $\nu$ is often referred to as the stationary distribution in the past literature.}
\begin{align}
	&\nu(s) = \lim_{T\to \infty}\frac{1}{T}\sum_{t=0}^{T-1}\mathbb{P}_{\pi}(s^t=s) \label{eq:niu}, \\ 
	&\rho(s,a) = \lim_{T\to \infty}\frac{1}{T}\sum_{t=0}^{T-1}\mathbb{P}_{\pi}(s^t=s,a^t=a) \label{eq:rho}, \\ 
	&q(s,a,s') = \lim_{T\to \infty}\frac{1}{T}\sum_{t=0}^{T-1}\mathbb{P}_{\pi}(s^t=s,a^t=a,s^{t+1} = s'). \label{eq:q}
\end{align}
Intuitively, $\nu(s), \rho(s,a)$ and $q(s,a,s')$ are the long-term average proportion of time of encountering the state $s$, state-action pair $(s,a)$, and state-action-next-state triple $(s,a,s')$, when executing policy $\pi$ in an MDP with transition probability kernel $P$. It can be readily shown that under Assumption \ref{ass:ergodic}, the limits in \eqref{eq:niu}, \eqref{eq:rho}, and \eqref{eq:q} indeed exist, and moreover, the following relations always hold:
\begin{align*}
	\rho(s,a) = \nu(s) \pi(a \mid s),\qquad q(s,a,s') = \rho(s,a) P(s'\mid s,a).
\end{align*}

In this work, we are primarily concerned with the occupancy measure $q(s,a,s')$ due to players not knowing their independent transition probability kernel $P_i$ in the stochastic game $\G$. In the following, we provide conditions that fully characterize the set of feasible occupancy measures $q$.


\begin{definition}\label{def:Delta}
	We define the polytope of feasible occupancy measures, denoted by $\Delta$, as 
	\begin{align}\label{eq:Delta}
		\Delta: = \bigg\{q\in [0, 1]^{|\mathcal{A}||\mathcal S|^2} : \ \ \ \sum_{s,a,s'}q(s,a,s') = 1,\quad\sum_{s',a}q(s',a,s) = \sum_{a,s'}q(s,a,s')\ \ \forall s\in \mathcal S\bigg\}. 
	\end{align}
	Moreover, for any $q \in \Delta$, we define its induced transition probability matrix $P^q$ and its induced stationary policy $\pi^q$ by
	\begin{align*}
		P^q(s'\mid s,a):= \frac{q(s,a,s')}{\sum_{s'}q(s,a,s')} \ \forall s,a,s', \ \ \ \ \ \ \pi^q(a\mid s):=\frac{\sum_{s'}q(s,a,s')}{\sum_{a',s'}q(s,a',s')}\ \forall s,a.
	\end{align*}
	In addition, for a transition probability matrix $P$ or a set of transition probability matrices $\CP$, we denote by $\Delta(P)$ and $\Delta(\CP)$ the sets of occupancy measures in $\Delta$ that are compatible with $P$ and $\CP$, respectively; that is, their induced transition matrices equal $P$ or belong to $\CP$:
\begin{equation}\nonumber
\Delta(P):=\{q\in\Delta:P^q=P\},
\qquad
\Delta(\CP):=\{q\in\Delta:P^q\in\CP\}.
\end{equation}
\end{definition}

The reason why \eqref{eq:Delta} provides the set of feasible occupancy measures is because any feasible occupancy measure $q$ should be a valid probability distribution, i.e., $\sum_{s,a,s'}q(s,a,s') = 1$, and moreover, for any state $s\in \mathcal S$,  the probability of entering it should equal to the probability of leaving it, i.e., $\sum_{s',a}q(s',a,s) = \sum_{a,s'}q(s,a,s')\ \forall s\in \mathcal S.$ Given the above definition, we have the following useful lemma from \citep{altman2021constrained}.

\begin{lemma}[\citep{altman2021constrained}]\label{lem:feasible}
	If a function $q:\mathcal S\times\mathcal A\times\mathcal S\to[0,1]$ belongs to the feasible occupancy polytope $\Delta$ defined in \eqref{eq:Delta}, then it is exactly the occupancy measure associated with its induced transition probability matrix $P^q$ and stationary policy $\pi^q$. Specifically, \eqref{eq:q} holds when policy $\pi^q$ is executed in an MDP with transition probability matrix $P^q$.
\end{lemma}

\subsection{A Dual Formulation}

We now provide a dual formulation of the original stochastic game $\G$ based on occupancy measures. It is shown in \citep{etesami2022learning} that, under the ergodicity Assumption~\ref{ass:ergodic}, the independence of the players' internal chains allows the payoff functions to admit a simple decomposable form in terms of the occupancy measures $\rho$. However, since we are interested in stochastic games with unknown transition probabilities, we first extend this formulation to the occupancy measures $q$. Specifically, suppose that each player $i$ follows a stationary policy $\pi_i$, and let $\nu_i$, $\rho_i$, and $q_i$ denote the corresponding occupancy measures defined in \eqref{eq:niu}, \eqref{eq:rho}, and \eqref{eq:q}, respectively, induced by kernel $P_i$ under policy $\pi_i$. Then, the expected average payoff of player $i$, $V_i(\pi_i,\pi_{-i})$, can be equivalently expressed as a function of $\rho_i$ and $\rho_{-i}$.
\begin{proposition}\label{prop:dual_formulation}
	Let Assumptions \ref{ass:indep} and \ref{ass:ergodic} hold, and assume that each player $i$ follows a stationary policy $\pi_i$. Let $q_i\in\Delta_i(P_i)$ denote the occupancy induced by $\pi_i$ under the true transition kernel $P_i$, with corresponding state--action marginal $\rho_i$. Then, each player's payoff can be expressed as
	\begin{align}\label{eq:linear-occupation-form2}
		V_i(\pi_i,\pi_{-i}) = V_i(q_i,q_{-i}):=\sum_{s,a}\prod_{j=1}^n\sum_{s'_j}q_j(s_j,a_j,s'_j)r_i(s,a)=\langle q_i, g_i(q_{-i}) \rangle,
	\end{align} 
	where $g_i(q_{-i})$ is defined to be a vector of dimension $|\mathcal A_i||\mathcal S_i|^2$ whose $(s_i,a_i,s'_i)$-th component is
	\begin{align}\label{eq:l_def}
		g_i(q_{-i})_{(s_i,a_i,s'_i)}:&=\sum_{s_{-i},a_{-i}}\prod_{j\neq i}\sum_{s'_j}q_j(s_j,a_j,s'_j) r_i(s,a)\cr 
        &=\sum_{s_{-i},a_{-i}}\prod_{j\neq i}\rho_j(s_j,a_j) r_i(s,a)\ \ \quad  \forall s_i,a_i,s'_i.
	\end{align}
	We note that $g_i(q_{-i})_{(s_i,a_i,s'_i)}$ is independent of $s'_i$, with the same value repeated across all $s'_i\in\mathcal S_i$.
\end{proposition}

Using Lemma~\ref{lem:feasible} and Proposition~\ref{prop:dual_formulation}, the problem of finding optimal stationary policies for the players reduces to that of finding their optimal feasible occupancy measures. In fact, using the dual formulation, we can define an equivalent \emph{virtual game} associated with the original stochastic game $\G$ as follows:
\begin{definition}\label{def:v_game}
	Define the feasible occupancy polytope for player $i$ by
	$$\Delta_i= \bigg\{q_i\in [0, 1]^{|\mathcal{A}_i||\mathcal S_i|^2}\!: \sum_{s_i,a_i,s'_i}q_i(s_i,a_i,s'_i)= 1,\  \ \sum_{s'_i,a_i}\!q_i(s'_i,a_i,s_i)= \sum_{s'_i,a_i}\!q_i(s_i,a_i,s'_i) \ \forall s_i\in \mathcal S_i\bigg\}.$$
 Moreover, let $\Delta_i(P_i) \subset \Delta_i$ be the set of feasible occupancy measures for player $i$ whose induced transition probability matrix $P^{q_i}$ is exactly $P_i$. The virtual game $\V=\big([n], q=(q_i)_i,\{V_i({q})\}_i\big)$ associated with the original stochastic game $\G $ is an $n$-player continuous-action static game, where the action of player $i$ is to choose an $q_i$ from its action set $\Delta_i (P_i)$, and its payoff is given by \eqref{eq:linear-occupation-form2}.
\end{definition} 

Ideally, each player would optimize over the true-kernel feasible set
$\Delta_i(P_i)$, on which the virtual payoff is the linear functional
$\langle q_i,g_i(q_{-i})\rangle$ in its own occupancy variable $q_i$.
Since $P_i$ is unknown, our proposed algorithm (Algorithm~\ref{alg:main})
instead maintains an iterate $\hat q_i$ that is feasible for some transition
probability matrix in a confidence set. In particular, for any profile
$q=(q_j)_j$ with $q_j\in\Delta_j(P_j)$ for all $j$, we have
$V_i(q_i,q_{-i})=\langle q_i,g_i(q_{-i})\rangle$.
For an arbitrary relaxed iterate $\hat q=(\hat q_i)_i$, its payoff is defined as the payoff of its induced policy under the \emph{true} kernels $\{P_j\}_{j\in[n]}$, namely,
\[
V_i(\hat q):=
V_i(\pi_i^{\hat q_i},\pi_{-i}^{\hat q_{-i}}).
\]
Thus, $\hat q_i$ denotes an algorithmic occupancy variable, whereas $q_i$
denotes the true occupancy measure induced by $\pi_i^{\hat q_i}$ under the true kernel $P_i$.
By Lemma~\ref{lem:feasible}, whenever $\hat q_i\in\Delta_i(P_i)$, its induced
transition matrix satisfies $P^{\hat q_i}=P_i$, and $\hat q_i$ is exactly the
occupancy measure induced by $\pi_i^{\hat q_i}$ under the true kernel $P_i$. Hence, in this
case, $\hat q_i=q_i$. This distinction will be used explicitly throughout the paper.

Finally, for the sake of exploration and to obtain high probability regret bounds, we consider the following definition of shrunk feasible occupancy polytope. 

\begin{definition}\label{def:shrunk}
Given $0<\delta_i<1/(|\mathcal S_i||\mathcal A_i|)$, define the \emph{shrunk} feasible occupancy polytope
\[
\Delta_{i,\delta_i}:=\Big\{q_i\in\Delta_i:\rho_i(s_i,a_i)=\sum_{s_i'}q_i(s_i,a_i,s_i')\ge\delta_i,\ \forall s_i,a_i\Big\}.
\]
For a kernel $P_i$ or a set $\mathcal P_i$ of kernels, $\Delta_{i,\delta_i}(P_i)$ and $\Delta_{i,\delta_i}(\mathcal P_i)$ denote the corresponding subsets of $\Delta_{i,\delta_i}$ whose induced kernel equals $P_i$ or belongs to $\mathcal P_i$, respectively.
\end{definition}

One advantage of working with the shrunk feasible occupancy sets is that the shrinkage constraint in Definition~\ref{def:shrunk} ensures persistent exploration while incurring only a small loss in the best-response value. In particular, for every $\epsilon>0$, $\delta_i$ can be chosen sufficiently small so that optimization over the shrunk feasible occupancy polytope is within $\epsilon$ of optimization over the full occupancy polytope (see Corollary~\ref{cor:cce}). Moreover, as we will see, the rectangular confidence sets maintained by Algorithm~\ref{alg:main} over the shrunk occupancy polytope preserve the computational tractability of the occupancy formulation. In particular, for any rectangular confidence set $\mathcal P_i$, the shrunk feasible occupancy set $\Delta_{i,\delta_i}(\mathcal P_i)$ admits an efficient representation, as shown in the following lemma.


\begin{lemma}\label{lem:confidence-convex}
Suppose a confidence set $\mathcal P_i$ contains the set of probability transition matrices $P_i$ that satisfy box constraints
$L_i(s_i,a_i,s'_i)\le P_i(s_i'\mid s_i,a_i)\le
U_i(s_i,a_i,s'_i)$. Then
$\Delta_{i,\delta_i}(\mathcal P_i)$ is a closed
convex polytope admitting an explicit linear description with
$O(|\mathcal A_i||\mathcal S_i|^2)$ variables and constraints.
\end{lemma}

\subsection{Local Cover Times}

To provide meaningful convergence-rate guarantees, we work with a \emph{local} cover-time parameter. Formally, for a fixed stationary policy $\pi_i$, let
\[
\tau_{{\rm cov},i}^{\pi_i}
:=\max_{s_i^0\in \mathcal{S}_i}\min\big\{t\ge0:\{S_i^0,\ldots,S_i^t\}=\mathcal S_i, \ S^0_i=s^0_i \big\}
\]
be the cover time of the Markov chain with transition matrix $P_i^{\pi_i}$. Define the expected cover time
\begin{align}\label{eq:local-cover-definition}
H_{i}
&:=\sup_{\pi_i\in\Pi_i}
\mathbb E[\tau_{{\rm cov},i}^{\pi_i}],
&
H&:=\max_{i\in[n]}H_{i},
\end{align}
where $\Pi_i$ is the compact set of stationary policies on $\mathcal S_i$. Thus, $H_i$ represents the worst-case expected time needed for player $i$ to visit all of its local states under any fixed stationary policy. We note that the local cover time is defined with respect to each player's local state space, which can be exponentially smaller than the joint state space. In particular, under Assumption~\ref{ass:ergodic}, it is easy to see that $H_i<\infty$ for all $i$, and hence $H<\infty$.\footnote{For every stationary policy $\pi_i$, Assumption~\ref{ass:ergodic} implies that $P_i^{\pi_i}$ is a finite irreducible Markov chain and hence has finite expected cover time. Since $\Pi_i$ is compact, the mapping $\pi_i\mapsto P_i^{\pi_i}$ is continuous, and expected hitting times are locally continuous over irreducible Markov chains, compactness yields a finite uniform bound over all $\pi_i$.}

\section{A Fully Online Decentralized Learning Algorithm}\label{sec:alg-design}

In this section, we describe our fully online algorithm for stochastic games with unknown independent chains. Each primitive interaction produces exactly one transition/reward observation and one mirror-descent update. Unlike the episodic setting, which has been a standard approach in the previous literature \cite{etesami2022learning,qin2024learning}, no player waits until all of its state--action pairs have been observed, there are no burn-in episodes, and no episode-end synchronization or coordination among the players is required. The transition kernels remain unknown and are estimated from the same single trajectory used for learning.

\subsection{Anytime confidence sets and one-sample updates}
To describe our algorithm, we first introduce the following notation and definitions. Player $i$ maintains cumulative counters that record, respectively, the numbers of visits to each state--action pair $(s_i,a_i)$ and each state--action--next-state triple $(s_i,a_i,s_i')$ up to time $t$ along the single realized trajectory of the game:   
\begin{align*}
&N_i^t(s_i,a_i):=\sum_{\ell<t}{\bf1}\{(S_i^{\ell},A_i^{\ell})=(s_i,a_i)\},\cr 
&M_i^t(s_i,a_i,s_i'):=\sum_{\ell<t}{\bf1}\{(S_i^{\ell},A_i^{\ell},S_i^{\ell+1})=(s_i,a_i,s_i')\}.
\end{align*}
Let $\bar P_i^t$ denote the \emph{empirical transition kernel} for player $i$ given by 
\[
\bar P_i^t(s_i'\mid s_i,a_i):=\frac{M_i^t(s_i,a_i,s_i')}{N_i^t(s_i,a_i)\vee1},
\]
where $\vee$ denotes the maximum of the quantities, and define the \emph{confidence radius} at time $t$ by 
\begin{align*}
\beta_i^t(s_i,a_i)=\sqrt{\frac{2\Lambda_t}{N_i^t(s_i,a_i)\vee1}},  
\end{align*}
where $\Lambda_t:=\log\!\big(\frac{64nA_{\max}S_{\max}^2(t+2)^2}{\gamma}\big)$,  $S_{\max}:=\max_{j\in[n]}|\mathcal S_j|, and   A_{\max}:=\max_{j\in[n]}|\mathcal A_j|$. For a finite horizon $T$ we write $\Lambda_T:=\max_{t\le T}\Lambda_t$. Moreover, at time $t$, we write
\begin{align}\label{eq:online-policy}
&\hat\rho_i^t(s_i,a_i)=\sum_{s_i'}\hat q_i^t(s_i,a_i,s_i'),\cr 
&\hat\nu_i^t(s_i)=\sum_{a_i}\hat\rho_i^t(s_i,a_i),\cr 
&\pi_i^t(a_i\mid s_i)=\frac{\hat\rho_i^t(s_i,a_i)}{\hat\nu_i^t(s_i)},
\end{align}
where we note that if $\hat{q}^t_i\in \Delta_{i,\delta_i}(\mathcal{P}^t_i)$, then $\hat\rho_i^t(s_i,a_i)\ge\delta_i$ and $\hat\nu_i^t(s_i)\le1$, and hence
\begin{align}\label{eq:action-floor}
\pi_i^t(a_i\mid s_i)\ge\delta_i,\qquad\forall s_i,a_i.
\end{align}
After observing the single reward $r_i^t:=r_i(S^t,A^t)$, we define the importance-weighted vector $\widehat{R}^t_i$ as
\begin{align}\label{eq:online-estimator}
\widehat R_i^t(s_i,a_i,s_i')=
\frac{r_i^t\,{\bf1}\{(S_i^t,A_i^t)=(s_i,a_i)\}}
{\hat\rho_i^t(s_i,a_i)}\qquad \forall s_i,a_i,s_i',
\end{align}
where we note that $\|\widehat R_i^t\|_\infty\le1/\delta_i$. The estimator $\widehat R_i^t$ converts the realized payoff $r_i^t$ into an estimate of the occupancy-payoff gradient $g_i^t:=g_i(q_{-i}^t)$, with the importance-weighted factor $1/\hat\rho_i^t(s_i,a_i)$ correcting for the fact that the payoff is observed only at the realized state--action pair $(S_i^t,A_i^t)$.

\begin{remark} 
To see the intuition behind $\widehat R_i^t$, suppose hypothetically that the joint state--action profile $(S^t,A^t)$ were drawn from the product stationary law $\bigotimes_j\rho_j^t$. Then, for any fixed $(s_i,a_i,s'_i)$,
$$
\E_{\bigotimes_j\rho_j^t}[\widehat R_i^t(s_i,a_i,s_i')]=\frac{1}{\hat\rho_i^t(s_i,a_i)}\E[r_i(S^t,A^t)\mathbf 1\{(S_i^t,A_i^t)=(s_i,a_i)\}]=\frac{\rho_i^t(s_i,a_i)}{\hat\rho_i^t(s_i,a_i)}g_i^t(s_i,a_i,s_i').
$$
Thus, if additionally $\hat\rho_i^t=\rho_i^t$ (which holds, in particular, when the transition kernel induced by $\hat q_i^t$ coincides with the true transition kernel $P_i$), then $\E[\widehat R_i^t]=g_i^t$. Hence, $\widehat R_i^t$ can be viewed as using the single realized payoff $r_i^t$ to construct a stochastic estimate of the occupancy-payoff gradient $g_i^t$.
\end{remark}

\subsection{Description of the Algorithm}\label{subsec:algorithm-intuition}

\begin{algorithm}[t]
\caption{Fully Online One-Sample Occupancy Mirror Descent for Player $i$}\label{alg:main}
\begin{algorithmic}[1]
\State \textbf{Input:} Shrinkage  $\delta_i$, stepsizes $\{\eta_t\}$, regularizer $h_i(q_i)=\frac12\|q_i\|_2^2$, and confidence level $\gamma$.
\State Player $i$ initializes $N_i^0(s_i,a_i)=0$ and $M_i^0(s_i,a_i,s_i')=0$ for all $s_i,a_i,s_i'$, sets $\mathcal P_i^0$ to the set of all transition kernels on $\mathcal S_i$, and chooses $\hat q_i^0\in\Delta_{i,\delta_i}(\mathcal P_i^0)$.
\For{$t=0,1,2,\ldots$}
    \State Player $i$ observes $S_i^t$ and samples $A_i^t\sim\pi_i^t(\cdot\mid S_i^t)$, where
    \Statex \hspace{1.5em}$\displaystyle
    \pi_i^t(a_i\mid s_i)
    =\frac{\hat\rho_i^t(s_i,a_i)}{\hat\nu_i^t(s_i)}.$

    \State Player $i$ observes $r_i^t=r_i(S^t,A^t)$ and forms
    \Statex \hspace{1.5em}$\displaystyle
    \widehat R_i^t(s_i,a_i,s_i')
    =
    \frac{r_i^t{\bf1}\{(S_i^t,A_i^t)=(s_i,a_i)\}}
    {\hat\rho_i^t(s_i,a_i)},
    \qquad \forall s_i,a_i,s_i'.$

    \State Player $i$ also observes its next state $S_i^{t+1}$ and updates its local counters:
    \Statex \hspace{1.5em}$\displaystyle
    N_i^{t+1}(s_i,a_i)
    =
    N_i^t(s_i,a_i)
    +{\bf1}\{(S_i^t,A_i^t)=(s_i,a_i)\},$
    \Statex \hspace{1.5em}$\displaystyle
    M_i^{t+1}(s_i,a_i,s_i')
    =
    M_i^t(s_i,a_i,s_i')
    +{\bf1}\{(S_i^t,A_i^t,S_i^{t+1})=(s_i,a_i,s_i')\}.$

    \If{$N_i^{t+1}(S_i^t,A_i^t)=2^k$ for some integer $k\ge 0$,}
        \State Player $i$ computes its empirical transition kernel
        \Statex \hspace{1.5em}$\displaystyle
        \bar P_i^{t+1}(s_i'\mid s_i,a_i)
        =
        \frac{M_i^{t+1}(s_i,a_i,s_i')}
        {N_i^{t+1}(s_i,a_i)\vee1},
        \qquad \forall s_i,a_i,s_i'.$

        \State Player $i$ computes the confidence radii
        \Statex \hspace{1.5em}$
        \beta_i^{t+1}(s_i,a_i)
        =
        \sqrt{\frac{2\Lambda_{t+1}}
        {N_i^{t+1}(s_i,a_i)\vee1}} \ \  \forall s_i,a_i, \ \quad 
        \Lambda_{t+1}
        =
        \log\!\left(
        \frac{64nA_{\max}S_{\max}^2(t+3)^2}
        {\gamma}
        \right).$

        \State Player $i$ updates its confidence set as
        \Statex \hspace{1.5em}$\displaystyle
        \mathcal P_i^{t+1}
        =
        \mathcal P_i^t\cap
        \Big\{
        P_i:
        \left|
        P_i(s_i'\mid s_i,a_i)
        -
        \bar P_i^{t+1}(s_i'\mid s_i,a_i)
        \right|
        \le
        \beta_i^{t+1}(s_i,a_i),
        \ \forall s_i,a_i,s_i'
        \Big\}.$
    \Else
        \State Player $i$ sets $\mathcal P_i^{t+1}=\mathcal P_i^t$.
    \EndIf

    \State Player $i$ performs the proximal occupancy update
    \Statex \hspace{1.5em}$\displaystyle
    \hat q_i^{t+1}
    =
    \argmax_{q_i\in\Delta_{i,\delta_i}(\mathcal P_i^{t+1})}
    \left\{
    \eta_t\langle q_i,\widehat R_i^t\rangle
    -
    D_{h_i}(q_i\Vert\hat q_i^t)
    \right\}.$
\EndFor
\end{algorithmic}
\end{algorithm}

Using the definitions in the preceding subsection, the proposed online algorithm is summarized in Algorithm~\ref{alg:main}. To describe the algorithm, it is useful to distinguish between the two learning tasks carried out simultaneously by each player. Player $i$ does not know its local controlled transition kernel $P_i$ and therefore must learn the rows $P_i(\cdot\mid s_i,a_i)$ from its observed local transitions. At the same time, it must adapt its strategic policy based on its realized payoff samples. Algorithm~\ref{alg:main} performs both tasks along the same uninterrupted trajectory.

The variable $\hat q_i^t$ is player $i$'s current confidence-feasible occupancy iterate, which induces both a transition kernel $\hat P_i^t:=P^{\hat q_i^t}\in\mathcal P_i^t$ and a policy $\pi_i^t:=\pi^{\hat q_i^t}$. In particular, its state--action marginal $\hat\rho_i^t$ determines the policy through~\eqref{eq:online-policy}, so the occupancy update implicitly induces the policy used to generate the online trajectory. The shrinkage constraint $\hat\rho_i^t(s_i,a_i)\ge\delta_i$ guarantees persistent local exploration through~\eqref{eq:action-floor} and, at the same time, keeps the importance weights uniformly bounded. At each time $t$, player $i$ samples $A_i^t$ from its current policy $\pi_i^t(\cdot\mid S^t_i)$ and observes only its realized payoff $r_i^t:=r_i(S^t,A^t)$ and next local state $S_i^{t+1}$. These observations serve two distinct purposes. The transition triple $(S_i^t,A_i^t,S_i^{t+1})$ is used to learn the unknown true kernel $P_i$ and refine the confidence set $\mathcal P_i^t$. In contrast, $(S_i^t,A_i^t,r_i^t)$ is used to construct the importance-weighted vector $\widehat R_i^t$ in~\eqref{eq:online-estimator}. The vector $\widehat R_i^t$, which serves as a stochastic proxy for the gradient $g_i^t:=g_i(q_{-i}^t)$ of the occupancy payoff function, is then used to move the occupancy iterate $\hat q_i^t$ in a payoff-improving direction. The updated iterate is projected onto the shrunk confidence-feasible occupancy polytope $\Delta_{i,\delta_i}(\mathcal P_i^{t+1})$, ensuring that its induced transition kernel remains consistent with the information currently available to player $i$ while preserving the exploration constraint. The new occupancy iterate $\hat q_i^{t+1}$ in turn induces the next policy through its state--action marginal, closing the online learning loop.

The transition-confidence set is deliberately refreshed on a slower, event-driven clock. For each local pair $(s_i,a_i)$, player $i$ continues to collect every transition sample but rebuilds the confidence set only when the corresponding visit count first reaches a new dyadic threshold $1,2,4,8,\ldots$. Equivalently, at time $t+1$, a refresh is triggered if, for some $(s_i,a_i)$ and integer $k\ge0$, we have $N_i^{t+1}(s_i,a_i)=2^k$. We call such a time a \emph{confidence update time}; it does not reset the trajectory, the counters, or the optimization iterate.

There are two reasons for this dyadic schedule. First, it incurs only a constant-factor statistical loss. Between two successive thresholds $2^k$ and $2^{k+1}$, a confidence radius based on the last threshold is at most a factor of $\sqrt{2}$ larger than the radius based on the current count. Thus, the usual $N^t_i(s_i,a_i)^{-1/2}$ estimation rate is preserved. Second, and more importantly for the analysis, changing the confidence set changes the feasible occupancy polytope. On intervals where this polytope is fixed, the ordinary proximal step gives the small-movement bound (Lemma~\ref{lem:movement}), which is used to couple the time-varying local chain to a frozen-policy chain. Dyadic refreshing ensures that each local counter causes only $\mathcal{O}(\log T)$ feasible-set changes by time $T$. Hence, the exceptional times at which the feasible polytope itself changes can be charged separately without affecting the $\widetilde{\mathcal{O}}(T^{-1/2})$ order of our final bound. Crucially, these confidence updates are entirely local and asynchronous. Player $i$ determines its confidence update times using only its own counters. Player $j$ may refresh its confidence set at a completely different primitive time, and neither player needs to know when the other player's refresh occurs. Thus, the algorithm operates on a fast clock, on which every player updates its occupancy iterate after every interaction, and a slower, local, event-driven clock, on which each player independently refines its transition kernel.

At a confidence update time, player $i$ constructs a new rectangular confidence region using all samples accumulated so far and intersects it with its previous confidence set. By Lemma~\ref{lem:confidence-convex}, the resulting shrunk occupancy set admits an explicit polynomial-size linear representation. Thus, the online mirror-descent update, which amounts to a Euclidean projection onto $\Delta_{i,\delta_i}(\mathcal P_i^t)$, can be carried out efficiently in polynomial time. As shown in the following lemma, on the \emph{anytime confidence event} used in the analysis, the true kernel $P_i$ belongs to every such confidence set; hence, the nested confidence sets remain nonempty while progressively refining the uncertainty about $P_i$. 

\begin{lemma}\label{lem:anytime-confidence}
Define the anytime confidence event $\mathcal E_{\rm conf}
:=
\left\{
P_i\in\mathcal P_i^t,\quad
\forall i\in[n],\ \forall t\ge0
\right\}$. Under Algorithm~\ref{alg:main}, we have $\mathbb P(\mathcal E_{\rm conf})\ge 1-\gamma/4$. Consequently, on $\mathcal E_{\rm conf}$, $\Delta_{i,\delta_i}(P_i)
\subseteq
\Delta_{i,\delta_i}(\mathcal P_i^t)\ \forall i, t\ge 0$.
\end{lemma}

\begin{remark}
Algorithm~\ref{alg:main} is fully decentralized in terms of the information used by each player. In particular, player $i$ maintains only its own local counters $N_i^t$ and $M_i^t$ and checks for a confidence update only when the count associated with its currently observed local state--action pair $(S_i^t,A_i^t)$ reaches a power of two. Player $i$ neither observes nor keeps track of the counters, states, actions, transition samples, or confidence-update events of any other player. Consequently, confidence updates of different players may occur at different primitive times, and no common confidence-update clock, joint-state counter, synchronization signal, or coordination among the players is required.
\end{remark} 

We conclude this section by outlining the main ideas behind the finite-time regret analysis in the next section. The regret is decomposed into a mirror-descent optimization error, an occupancy-tracking error, and a one-sample gradient-estimation error. These terms are controlled using the proximal inequality, local coverage and confidence-set concentration, and a delayed coupling argument relating the time-varying trajectory to frozen-policy stationary dynamics, respectively. Finally, confidence-update times are handled separately by exploiting their logarithmic frequency under the dyadic schedule. Together, these ingredients establish the no-regret guarantee in the online setting.

\section{Finite-Time Regret Guarantee for the Fully Online Decentralized  Algorithm}\label{sec:finite_analysis}

We begin the analysis of Algorithm \ref{alg:main} by introducing the following terminology and filtration. 

\subsection{True and Confidence-Feasible Occupancies}

We distinguish throughout between \emph{true-kernel feasible} and
\emph{confidence-feasible} occupancy measures. An occupancy measure
$q_i$ is called \emph{true-kernel feasible} if
$q_i\in\Delta_i(P_i)$, i.e., if its induced transition kernel $P^{q_i}$ equals the true transition kernel $P_i$. Its corresponding state--action and state
marginal occupancies are denoted by $\rho_i$ and $\nu_i$, respectively. In contrast, an occupancy measure $\hat q_i$ is called
\emph{confidence-feasible} if
$\hat q_i\in\Delta_{i,\delta_i}(\mathcal P_i)$, meaning that its
induced transition kernel belongs to the confidence set
$\mathcal P_i$, i.e., $P^{\hat q_i}\in\mathcal P_i$. Its corresponding marginals are denoted by $\hat\rho_i$ and $\hat\nu_i$.

Accordingly, for each time $t$, let $q_i^t,\rho_i^t,\nu_i^t$ denote the
true-kernel feasible occupancies induced by the current stationary
policy $\pi_i^t$ and the true kernel $P_i$. Thus, $q_i^t$ is the
true-kernel feasible occupancy induced by $(\pi_i^t,P_i)$, whereas $\hat q_i^t$ is the confidence-feasible occupancy that induces $(\pi_i^t,\hat P_i^t)$, where, for simplicity, we write $\hat P_i^t:=P^{\hat q_i^t}\in\mathcal P_i^t$. Hence, throughout the analysis, unhatted occupancy quantities refer
to the true kernel $P_i$, while hatted occupancy quantities refer to
the confidence-feasible iterates maintained by the algorithm.

\subsection{Filtration and Regret Decomposition}

For the probabilistic arguments, let the filtration $\mathcal{F}_t$ contain all the information available strictly before time $t$, such as $S^0,A^0,r^0,S^1,A^1,r^1,\ldots,
A^{t-1},r^{t-1}$, as well as the current state $S^t$, but not the fresh randomization used to sample $A^t$. Note that this is a global filtration used only for the purpose of analysis and does not enlarge any player's information set in Algorithm~\ref{alg:main}. Thus, $\mathcal F_t$ represents the information available immediately before the action $A^t$ is sampled: the current state $S^t$ and all states, actions, rewards, transition observations, counters, confidence sets, and iterates constructed from primitive times strictly before $t$ are $\mathcal F_t$-measurable. After conditioning on $\mathcal F_t$, the action $A^t$ is sampled according to the current policy profile, the reward $r^t$ is observed, and the next state $S^{t+1}$ is generated; consequently, these variables are measurable with respect to $\mathcal F_{t+1}$. Throughout the paper, expressions such as $\mathbb P(S_i^t=\cdot\mid\mathcal F_{t-D})$ denote the conditional probability distribution of $S_i^t$ given the information available at the beginning of the delay time $t-D$.

\begin{definition}
For player $i$ and a fixed true-kernel comparator
$q_i'\in\Delta_{i,\delta_i}(P_i)$, define the fixed-comparator regret $\operatorname{Reg}_i(T;q_i')
:=
\sum_{t=0}^{T-1}
\bigl[V_i(q_i',q_{-i}^t)-V_i(q_i^t,q_{-i}^t)\bigr]$, where $q_i^t\in\Delta_i(P_i)$ is the true-kernel feasible occupancy
induced by the current policy
$\pi_i^t:=\pi_i^{\hat q_i^t}$ under the true kernel $P_i$.
\end{definition}

We note that since $V_i(q_i,q_{-i})=\langle q_i,g_i(q_{-i})\rangle$ whenever the displayed occupancies are true-kernel feasible, using Proposition~\ref{prop:dual_formulation}, we have 
\begin{align}\label{eq:player-fixed-regret}
\operatorname{Reg}_i(T;q_i')
=\sum_{t=0}^{T-1}\bigl[V_i(q_i',q_{-i}^t)-V_i(q_i^t,q_{-i}^t)\bigr]=\sum_{t=0}^{T-1}\langle q'_i-q_i^t,g_i(q_{-i}^t)\rangle.   
\end{align}
Writing $g_i^t:=g_i(q_{-i}^t)$ and adding and subtracting $\hat q_i^t$ and $\widehat R_i^t$ to \eqref{eq:player-fixed-regret} gives the basic decomposition
\begin{align}\label{eq:master-regret-decomposition}
\operatorname{Reg}_i(T;q_i')
={}&\underbrace{\sum_{t<T}\langle q_i'-\hat q_i^t,\widehat R_i^t\rangle}_{\mathcal M_i(T;q_i')}
+\underbrace{\sum_{t<T}\langle q_i'-\hat q_i^t,g_i^t-\widehat R_i^t\rangle}_{\mathcal E_i(T;q_i')}
+\underbrace{\sum_{t<T}\langle\hat q_i^t-q_i^t,g_i^t\rangle}_{\mathcal Q_i(T)}.
\end{align}

The decomposition~\eqref{eq:master-regret-decomposition} provides the roadmap for the proof. The first term, $\mathcal M_i$, is the optimization error that would remain even if the payoff-gradient samples were unbiased and the transition kernels were known. It is controlled directly by the Euclidean proximal inequality in Lemma~\ref{lem:finite-md}. The third term, $\mathcal Q_i$, captures the error from optimizing over the confidence-feasible occupancy $\hat q_i^t$ rather than the true-kernel occupancy $q_i^t$ induced by the same policy. To control this term, in Lemma~\ref{lem:movement}, we first show that mirror descent changes the policy slowly whenever the confidence set remains unchanged. We then freeze the policy over deterministic local-cover blocks and couple the actual local chain with the corresponding frozen chain. This yields linear growth of every local state--action counter, which in turn implies shrinking confidence radii for $P_i$. Uniform contraction then translates the kernel estimation error into the corresponding stationary occupancy error, which controls $\mathcal Q_i$. The second term, $\mathcal E_i$, captures the main difficulty arising from the fully online setting: $\widehat R_i^t$ is formed from a single payoff sample generated by an adaptive, nonstationary Markov trajectory, whereas $g_i^t$ is the stationary payoff gradient associated with the current true-kernel occupancies $q_{-i}^t$ induced by the current policy $\pi^t$. We control this term using delayed analysis in Lemma~\ref{lem:delay-coupling}.

\subsection{Main Regret Bound Analysis}

Here, we provide the main analysis for the finite-time no-regret guarantee. We begin with the following lemma, which bounds the first term $\mathcal M_i(T;q_i')$ in the regret decomposition~\eqref{eq:master-regret-decomposition}. 

\begin{lemma}\label{lem:finite-md}
On the anytime confidence event $\mathcal E_{\rm conf}$, for every $q_i'\in\Delta_{i,\delta_i}(P_i)$ and every constant stepsize $\eta>0$, we have
\begin{align}\label{eq:online-md-regret}
\sum_{t<T}\langle q_i'-\hat q_i^t,\widehat R_i^t\rangle
\le \frac{\|q_i'-\hat q_i^0\|_2^2}{2\eta}
+\frac\eta2\sum_{t<T}\|\widehat R_i^t\|_2^2
\le \frac1\eta+\frac{|\mathcal S_i|\eta T}{2\delta_i^2}.
\end{align}
\end{lemma}

The following lemma provides a high-probability bound on the policy and occupancy movements over the course of the algorithm, which allows us to bound the third term $\mathcal Q_i(T)$ in~\eqref{eq:master-regret-decomposition}.

\begin{lemma}\label{lem:movement}
Fix $T$ and use the constant stepsize $\eta=T^{-1/2}$. Define
$L_i:=\sqrt{|\mathcal S_i|^{3}|\mathcal A_i|}$. Suppose
\begin{align}\label{eq:T-cover-threshold}
\max_i\frac{L_iH_i^2\eta}{\delta_i^2}\le\frac{1}{16}.
\end{align}
Then, with probability at least $1-\gamma/2$, the following hold
simultaneously for every player $i$:
\begin{enumerate}[(i)]
\item On every step $t$ at which player $i$ does not update its
confidence set,
\begin{align}\label{eq:pi-movement}
\max_{s_i\in\mathcal S_i}
\|\pi_i^{t+1}(\cdot\mid s_i)-\pi_i^t(\cdot\mid s_i)\|_1
\le\frac{2L_i\eta}{\delta_i^2}.
\end{align}

\item Let $K_i(T):=
c\frac{H_i}{\delta_i}
\big(\frac{\Lambda_T}{\delta_i}+|\mathcal S_i||\mathcal A_i|
\log_2 T\big)$. Then, for every $K_i(T)\le t<T$,
\begin{align}\label{eq:pointwise-occ-error}
\|\hat\rho_i^t-\rho_i^t\|_1
\le c\frac{|\mathcal{S}_i|}{\kappa}\max_{s_i,a_i}\beta^t(s_i,a_i)\leq 
c\frac{|\mathcal S_i|}{\kappa}
\sqrt{\frac{H_i\Lambda_t}{\delta_i t}}.
\end{align}

\item Consequently,
\begin{align}\label{eq:Q-bound}
\Big|\frac{1}{T}\sum_{t<T}
\langle\hat q_i^t-q_i^t,g_i^t\rangle\Big|
\le
c\frac{K_i(T)}{T}
+c\frac{|\mathcal S_i|}{\kappa}
\sqrt{\frac{H_i\Lambda_T}{\delta_iT}}.
\end{align}
\end{enumerate}
\end{lemma}

Next, we bound the second term $\mathcal E_i(T;q_i')=\sum_{t<T}\langle q_i'-\hat q_i^t,g_i^t-\widehat R_i^t\rangle$ in the regret decomposition. At time $t$, the estimator $\widehat R_i^t$ uses the actual observation $(S^t,A^t,r_i^t)$,\footnote{The actual trajectory is generated by the true kernel $P_i$ with the interaction flow
$\mathcal P_i^t \rightarrow \hat q_i^t \rightarrow \pi_i^t \rightarrow A_i^t \xrightarrow{P_i} S_i^{t+1}$.
Thus, $\hat P_i^t=P^{\hat q_i^t}$ is used only for analysis to relate the confidence-feasible occupancy $\hat q_i^t$, induced by $(\pi_i^t,\hat P_i^t)$, to the true-kernel occupancy $q_i^t$, induced by $(\pi_i^t,P_i)$, and does not play a role in generating the actual trajectory.} whereas $g_i^t=g_i(q_{-i}^t)$ is defined in terms of the stationary occupancies $q_j^t$. In general, there is no reason for the actual state $S^t$, conditional on the immediately preceding history, to be distributed according to these stationary occupancies. We therefore look back $D_T$ steps and condition on $\mathcal F_{t-D_T}$, where the delay $D_T$ is introduced for one fundamental reason: to provide a sufficiently long window over which the actual time-varying process can be approximated by a time-homogeneous process that has enough time to mix, so that the conditional law of the actual online sample can be well approximated by the corresponding stationary-occupancy distribution. Since the policies continue to change during these $D_T$ steps, the proof couples the actual trajectory with an auxiliary time-homogeneous trajectory whose policies are frozen at time $t-D_T$. This allows us to show that, over a time window $[t-D_T,t]$ during which no player performs a confidence update, the conditional law of the actual state--action profile at time $t$ is close to the product stationary law corresponding to the policies at time $t-D_T$.

\begin{lemma}\label{lem:delay-window-mixing}
Let $D_T\ge 0$ be a nonnegative integer delay, and fix $t\ge D_T$. Assume that no player updates its confidence set during the time delay window $[t-D_T,t]$, and let $L_j\!=\!\sqrt{|\mathcal A_j||\mathcal S_j|^{3}}$. Then
\begin{align}\label{eq:actual-product-mixing}
\bigg\|
\mathbb P((S^t,A^t)\in\cdot\mid\mathcal F_{t-D_T})
-
\bigotimes_{j=1}^n\rho_j^{t-D_T}
\bigg\|_{\rm TV}
\le
cD_T^2\eta\sum_{j=1}^n\frac{L_j}{\delta_j^2}
+
n(1-\kappa)^{D_T}.
\end{align}
\end{lemma}

To relate the payoff estimator $\widehat R_i^t$, which is constructed from the actual nonstationary trajectory, to the occupancy-payoff gradient $g_i^{t}$, we use a delayed comparison over a mixing window of length $D_T$. In particular, by conditioning on $\mathcal F_{t-D_T}$, we leave $(S^t,A^t)$ random over the mixing window. We first use the slow movement of the occupancy iterates in Lemma~\ref{lem:movement} to replace the current denominator $\hat\rho_i^t$ in $\widehat R_i^t$ with the delayed denominator $\hat\rho_i^{t-D_T}$, which is $\mathcal F_{t-D_T}$-measurable. Lemma~\ref{lem:delay-window-mixing} then shows that the conditional law of $(S^t,A^t)$ given $\mathcal F_{t-D_T}$ is close to the product stationary law $\bigotimes_j\rho_j^{t-D_T}$. Under this product stationary law, the mean of $\widehat R_i^t$ with denominator $\hat\rho_i^{t-D_T}$ is $g_i^{t-D_T}(s_i,a_i,s_i')\rho_i^{t-D_T}(s_i,a_i)/\hat\rho_i^{t-D_T}(s_i,a_i)$. The occupancy-estimation bound controls the difference between this quantity and the delayed occupancy-payoff gradient $g_i^{t-D_T}$, while the slow movement of the policies controls the remaining drift from $g_i^{t-D_T}$ to $g_i^t$. Using these observations, the following lemma bounds the second term $\mathcal E_i(T;q_i')$ in~\eqref{eq:master-regret-decomposition} as follows:

\begin{lemma}\label{lem:delay-coupling}
Let $g_i^t:=g_i(q_{-i}^t)$ be the occupancy-payoff gradient \eqref{eq:l_def} and $D_T:=1+\lceil\frac{1}{\kappa}\log (\frac{8nT^3}{\gamma})\rceil$. Then, with probability at least $1-\gamma$, uniformly over $q_i'\in\Delta_{i,\delta_i}(P_i)$, we have
\begin{align}\label{eq:cumulative-gradient-error}
&\!\!\!\!\frac1T\Big|\sum_{t<T}\langle q_i'-\hat q_i^t,\widehat R_i^t-g_i^t\rangle\Big|
\lesssim{}\frac{\sqrt{TD_T\Lambda_T}+D_T\Lambda_T}{T\delta_i}
+\frac{D_T\eta L_i}{\delta_i^3}
+\frac{D_T^2\eta}{\delta_i}\sum_{j=1}^n\frac{L_j}{\delta_j^2}\cr
&\qquad\qquad+\frac{|\mathcal S_i|}{\kappa\delta_i}
\sqrt{\frac{H_i\Lambda_T}{\delta_iT}}
+D_T\eta\sum_{j\ne i}\frac{L_j}{\kappa\delta_j^2}+\frac{D_T}{T\delta_i}\sum_{j=1}^nJ_j(T)
+\frac{D_T+K_i(T)}{T\delta_i},
\end{align}
where $L_j=\sqrt{|\mathcal S_j|^3|\mathcal A_j|}$,  $K_i(T)=
c\frac{H_i}{\delta_i}
\big(\frac{\Lambda_T}{\delta_i}
+|\mathcal S_i||\mathcal A_i|\log_2 T\big)$, and $J_i(T)=
|\mathcal S_i||\mathcal A_i|(\log_2 T+1)$.
\end{lemma}


Finally, by combining the preceding lemmas we obtain the main result of this section.

\begin{theorem}\label{theo:finite}
Let Assumptions~\ref{ass:indep} and~\ref{ass:ergodic} hold. Fix
$T\ge2$ and $\gamma\in(0,1)$, and assume players follow Algorithm~\ref{alg:main}
with the quadratic regularizer and constant stepsize
$\eta=T^{-1/2}$. Let
\[
S_{\max}:=\max_{i\in[n]}|\mathcal S_i|,
\qquad
A_{\max}:=\max_{i\in[n]}|\mathcal A_i|,
\qquad
H:=\max_{i\in[n]}H_i,
\qquad
\delta:=\min_{i\in[n]}\delta_i, 
\]
and suppose $T\!\ge\! \frac{256S_{\max}^3A_{\max}H^4}{\delta^4}$,
so that condition~\eqref{eq:T-cover-threshold} holds.\footnote{For smaller horizons the criterion is bounded by
its trivial constant upper bound.} Then, with probability at least
$1-\gamma$,
\begin{align}\label{eq:finite_bound}
\max_{i\in[n]}\;
\max_{q_i'\in\Delta_{i,\delta_i}(P_i)}
\frac{1}{T}\operatorname{Reg}_i(T;q_i')
\le
\frac{\mathfrak C_1}{\sqrt T}+\frac{\mathfrak C_2}{T},
\end{align}
where, up to universal constants and logarithmic factors,
\begin{align}\label{eq:CT}
\mathfrak C_1
&=
\widetilde{\mathcal O}\bigg(
\frac{S_{\max}\sqrt{H}}
{\kappa\delta^{3/2}}
\sqrt{\log\!\Big(\frac{nA_{\max}S_{\max}^2T^2}{\gamma}\Big)}
+
\frac{nS_{\max}^{3/2} A_{\max}^{1/2}}
{\kappa^2\delta^3}
\log^2\!\Big(\frac{nT^3}{\gamma}\Big)
\bigg),\cr
\mathfrak C_2
&=
\widetilde{\mathcal O}\left(
\frac{H}{\delta^3}
\log\!\Big(\frac{nA_{\max}S_{\max}^2T^2}{\gamma}\Big)
+
\frac{H S_{\max}A_{\max}}{\delta^2}
+
\frac{nS_{\max}A_{\max}}{\kappa \delta}
\log\!\Big(\frac{nT^3}{\gamma}\Big)
\right).
\end{align}
\end{theorem}

\begin{proof}
Fix a player $i$ and an arbitrary comparator
$q_i'\in\Delta_{i,\delta_i}(P_i)$. By
\eqref{eq:master-regret-decomposition},
\begin{align*}
\frac{1}{T}\operatorname{Reg}_i(T;q_i')
={}&
\frac{\mathcal M_i(T;q_i')}{T}
+\frac{\mathcal E_i(T;q_i')}{T}
+\frac{\mathcal Q_i(T)}{T}.
\end{align*}
We bound the three terms using the preceding lemmas.

First, on the anytime confidence event $\mathcal E_{\rm conf}$, Lemma~\ref{lem:finite-md}
with $\eta=T^{-1/2}$ gives
\begin{align}\label{eq:finite-proof-md}
\frac{\mathcal M_i(T;q_i')}{T}
\le
\frac{1}{\sqrt T}
+\frac{|\mathcal S_i|}{2\delta_i^2\sqrt T}.
\end{align}
Second, Lemma~\ref{lem:delay-coupling}, uniformly over
$q_i'\in\Delta_{i,\delta_i}(P_i)$, gives
\begin{align}\label{eq:finite-proof-gradient}
\frac{|\mathcal E_i(T;q_i')|}{T}
\lesssim{}&
\frac{\sqrt{TD_T\Lambda_T}+D_T\Lambda_T}{T\delta_i}
+\frac{D_TL_i}{\delta_i^3\sqrt T}
+\frac{D_T^2}{\delta_i\sqrt T}
\sum_{j=1}^n\frac{L_j}{\delta_j^2}\cr
&+
\frac{|\mathcal S_i|}{\kappa\delta_i}
\sqrt{\frac{H_i\Lambda_T}{\delta_iT}}
+\frac{D_T}{\sqrt T}
\sum_{j\ne i}\frac{L_j}{\kappa\delta_j^2}
+\frac{D_T}{T\delta_i}\sum_{j=1}^nJ_j(T)
+\frac{D_T+K_i(T)}{T\delta_i}.
\end{align}
Finally, part~(iii) of Lemma~\ref{lem:movement} yields
\begin{align}\label{eq:finite-proof-Q}
\frac{|\mathcal Q_i(T)|}{T}
\lesssim
\frac{K_i(T)}{T}
+
\frac{|\mathcal S_i|}{\kappa}
\sqrt{\frac{H_i\Lambda_T}{\delta_iT}}.
\end{align}

We now simplify these bounds in terms of the maximal individual-chain
parameters appearing in the theorem. By definition, $|\mathcal S_i|\le S_{\max}$, $|\mathcal A_i|\le A_{\max}$, $H_i\le H$, $\delta_i\ge\delta$, and therefore
\begin{align*}
L_i&\le S_{\max}^{3/2}A_{\max}^{1/2},
\quad
\sum_{j=1}^n\frac{L_j}{\delta_j^2}
\le
\frac{nS_{\max}^{3/2}A_{\max}^{1/2}}{\delta^2},
\quad
\!\Lambda_T
\lesssim
\log\!\Big(\frac{nA_{\max}S_{\max}^2T^2}{\gamma}\Big),
\quad
\!\!D_T
\lesssim
\frac{1}{\kappa}
\log\!\Big(\frac{nT^3}{\gamma}\Big).
\end{align*}
Substituting these relations into
\eqref{eq:finite-proof-md}--\eqref{eq:finite-proof-Q}, the dominant
terms of order $T^{-1/2}$ are
\begin{align}\label{eq:finite-proof-leading}
\frac{1}{\sqrt T}\,
\widetilde{\mathcal O}\bigg(
\frac{S_{\max}\sqrt{H}}
{\kappa\delta^{3/2}}
\sqrt{\log\!\Big(\frac{nA_{\max}S_{\max}^2T^2}{\gamma}\Big)}
+
\frac{nS_{\max}^{3/2}A_{\max}^{1/2}}
{\kappa^2\delta^3}
\log^2\!\Big(\frac{nT^3}{\gamma}\Big)
\bigg).
\end{align}
Indeed, the first displayed term contains the occupancy-estimation
contributions in \eqref{eq:finite-proof-gradient} and
\eqref{eq:finite-proof-Q}, while the second contains the dominant
$D_T^2\sum_jL_j/\delta_j^2$ contribution in
\eqref{eq:finite-proof-gradient}; the remaining $T^{-1/2}$ terms are
of no larger order and are absorbed into these two terms. For the lower-order terms, using
\begin{align*}
J_j(T)
&=
|\mathcal S_j||\mathcal A_j|(\log_2T+1),
\qquad
K_i(T)
=
c\frac{H_i}{\delta_i}
\left(
\frac{\Lambda_T}{\delta_i}
+
|\mathcal S_i||\mathcal A_i|\log_2T
\right),
\end{align*}
we obtain
\begin{align}\label{eq:finite-proof-transient}
\frac{1}{T}\,
\widetilde{\mathcal O}\left(
\frac{H}{\delta^3}
\log\!\Big(\frac{nA_{\max}S_{\max}^2T^2}{\gamma}\Big)
+
\frac{HS_{\max}A_{\max}}{\delta^2}
+
\frac{nS_{\max}A_{\max}}{\kappa\delta}
\log\!\Big(\frac{nT^3}{\gamma}\Big)
\right).
\end{align}
Here, the logarithmic factors in $T$ arising from $J_j(T)$ and the
second term in $K_i(T)$ are absorbed into the
$\widetilde{\mathcal O}$ notation, while the dependence on the
confidence level $\gamma$ is displayed explicitly.

Combining \eqref{eq:finite-proof-leading} and
\eqref{eq:finite-proof-transient}, and then taking the maximum over
$q_i'\in\Delta_{i,\delta_i}(P_i)$ and $i\in[n]$, gives
\eqref{eq:finite_bound} with $\mathfrak C_1$ and $\mathfrak C_2$ as
defined in \eqref{eq:CT}. Finally, Lemma~\ref{lem:delay-coupling} holds on an event of
probability at least $1-\gamma$ that includes the event of
Lemma~\ref{lem:movement}, and the anytime confidence event required
by Lemma~\ref{lem:finite-md} is also included in that event.
Consequently, all the preceding bounds hold simultaneously with
probability at least $1-\gamma$, which completes the proof.
\end{proof}

\vspace{-0.3cm}
\begin{corollary}\label{cor:cce}
On the event of Theorem~\ref{theo:finite}, for sufficiently small
action floor $\delta=\max_i\delta_i$, the empirical distribution over
the policy profiles $\{\pi^t\}_{t=0}^{T-1}$ is an
$(\epsilon_T+\epsilon_{\delta})$-coarse correlated
equilibrium of the original stochastic game, where
$\epsilon_T:=
\frac{\mathfrak C_1}{\sqrt T}
+
\frac{\mathfrak C_2}{T}$, and $\epsilon_{\delta}>0$ is the approximation error due to
the action-floor constraints, which can be made arbitrarily small as
$\delta\to0$.
\end{corollary}

\begin{proof}
Let $U$ be uniformly distributed over $\{0,\ldots,T-1\}$. For every
player $i$, the definition of $\operatorname{Reg}_i(T;q_i')$ and
Theorem~\ref{theo:finite} give
\[
\max_{q_i'\in\Delta_{i,\delta_i}(P_i)}
\E\!\left[
V_i(q_i',q_{-i}^U)-V_i(q_i^U,q_{-i}^U)
\right]
=
\max_{q_i'\in\Delta_{i,\delta_i}(P_i)}
\frac1T\sum_{t=0}^{T-1}
\left[
V_i(q_i',q_{-i}^t)-V_i(q_i^t,q_{-i}^t)
\right]
\le \epsilon_T.
\]
Moreover, for every $\epsilon_{\delta}>0$, the action floor $\delta$
can be chosen sufficiently small so that optimization over the shrunk
feasible occupancy polytope is within $\epsilon_{\delta}$ of
optimization over the full occupancy polytope
\cite[Lemma 3.4]{etesami2022learning}. Hence,
\[
\max_{q_i'\in\Delta_i(P_i)}
\frac1T\sum_{t=0}^{T-1}V_i(q_i',q_{-i}^t)
\le
\max_{q_i'\in\Delta_{i,\delta_i}(P_i)}
\frac1T\sum_{t=0}^{T-1}V_i(q_i',q_{-i}^t)
+\epsilon_{\delta}.
\]
Combining the two bounds yields $\max_{q_i'\in\Delta_i(P_i)}
\E\!\left[
V_i(q_i',q_{-i}^U)-V_i(q_i^U,q_{-i}^U)
\right]
\le
\epsilon_T+\epsilon_{\delta}, \ \forall i\in[n]$, which is precisely the $(\epsilon_T+\epsilon_{\delta})$-coarse
correlated equilibrium condition for the virtual occupancy game.
By the equivalence between the virtual occupancy game and the original
stochastic game, the empirical distribution uniformly supported on
the corresponding stationary policy profiles
$\{\pi^t\}_{t=0}^{T-1}$ is therefore an
$(\epsilon_T+\epsilon_{\delta})$-coarse correlated equilibrium of the
original stochastic game.
\end{proof}

\begin{remark}
Theorem~\ref{theo:finite} uses a horizon-dependent constant stepsize
only to make the finite-time regret bound transparent. This dependence
on a known horizon can be removed by a standard doubling schedule:
run the algorithm over epochs of lengths $1,2,4,\ldots$ and use
$\eta_m\asymp2^{-m/2}$ during epoch $m$, without resetting the
trajectory or the confidence counters. Applying the preceding
finite-time analysis epochwise after the local-cover threshold and
summing the resulting regret bounds yields the same
$\widetilde{\mathcal O}(T^{-1/2})$ average-regret rate, since the sum
of the $\sqrt{T_m}$ contributions over geometrically increasing epoch
lengths is of order $\sqrt T$, while the lower-order transient
contributions remain logarithmic in $T$. Alternatively, one may use
$\eta_t=c/\sqrt{t+1}$; once the stepsize is small enough for
\eqref{eq:T-cover-threshold} to hold, the same block comparison
applies, while the finitely many preceding time steps contribute only
an additive $O(1/T)$ term to the average regret.
\end{remark}

\section{Asymptotic Last-Iterate Convergence to an $\epsilon$-Nash Equilibrium}\label{sec:asymptotic}

It is known that, for general reward functions, computing a NE in stochastic games with independent chains is PPAD-hard even when the transition kernels are known \citep[Proposition 3.3]{etesami2022learning}. Therefore, under general reward functions, developing a scalable no-regret algorithm for obtaining a coarse correlated equilibrium, as established in the previous section, is perhaps the best one can hope for computationally. In this section, we show that if the reward functions possess additional structure, in the sense that the game admits a \emph{stable} NE, then the same fully online algorithm converges with arbitrarily high probability to a stationary $\epsilon$-NE of the original stochastic game. This complements the results of the previous section: not only does the algorithm provide a scalable no-regret guarantee for coarse correlated equilibria, but under the additional stability condition, the same algorithm converges to an $\epsilon$-NE of the original stochastic game. To this end, we first introduce the following notion of equilibrium stability, which has been well studied and motivated for static games \citep{mertikopoulos2019learning}.   

\begin{assumption}\label{ass:g_stable}
The virtual game $\mathcal V$ admits a unique NE $q^*$ satisfying
\[
\sum_{i=1}^n\langle g_i(q_{-i}),q_i^*-q_i\rangle\ge0,
\qquad \forall q\in\boldsymbol\Delta_{\boldsymbol\delta}(\boldsymbol P):=(\Delta_{\delta_1}(P_1),\ldots,\Delta_{\delta_n}(P_n)),
\]
with equality only at $q=q^*$.
\end{assumption}

Assumption~\ref{ass:g_stable} requires that the joint payoff-gradient field always points toward the unique NE $q^*$. Moreover, the inequality is strict away from $q^*$, so $q^*$ is the unique stable point toward which the learning dynamics are driven. 

We now turn to the convergence analysis of the fully online algorithm under Assumption~\ref{ass:g_stable}. A key difficulty is that the transition kernels are unknown and the confidence-feasible occupancy sets continue to evolve over time. We first show in Lemma~\ref{lem:asymptotic-visitation} that persistent exploration guarantees a positive asymptotic visitation frequency for every local state--action pair. This implies that the confidence radii vanish and, consequently, that the confidence-feasible occupancies $\hat q_i^t$ asymptotically approach their true-kernel counterparts $q_i^t$. We next show in Lemma~\ref{lem:restart-jumps} that the confidence-feasible occupancy sets converge to their corresponding true-kernel feasible sets and that the occupancy jumps caused by confidence-set updates are summable. Finally, in Theorem~\ref{theo:assymptotic}, we combine these properties with the diminishing stepsizes and the variational-stability condition in Assumption~\ref{ass:g_stable} to establish convergence of the occupancy iterates to the unique stable equilibrium $q^*$, which in turn yields last-iterate convergence of the induced policies to a stationary $\epsilon$-NE of the original game.

\begin{lemma}\label{lem:asymptotic-visitation}
Suppose Assumptions~\ref{ass:indep} and~\ref{ass:ergodic} hold, and let
$\eta_t=\eta_0(t+1)^{-a}$ with $a\in(1/2,1]$. On the anytime confidence event $\mathcal E_{\rm conf}$, almost surely, for every player $i$ and every local pair $(s_i,a_i)$,
\begin{align}\label{eq:asymptotic-count}
\liminf_{t\to\infty}\frac{N_i^t(s_i,a_i)}{t}
\ge \frac{\delta_i}{32H_i}>0.
\end{align}
Consequently, almost surely, there exists a finite random time $T_0$
such that, for every $t\ge T_0$ and $i$,
\begin{align}
&\max_{s_i,a_i}\beta_i^t(s_i,a_i)
=O\bigg(\!\sqrt{\frac{\log(t+2)}{t+1}}\bigg),
\label{eq:asymptotic-beta}\\
&\|\hat q_i^t-q_i^t\|_1
=O\bigg(\!\sqrt{\frac{\log(t+2)}{t+1}}\bigg)\to 0.
\label{eq:asymptotic-q}
\end{align}
\end{lemma}

Next, we use the vanishing confidence radii established in Lemma~\ref{lem:asymptotic-visitation} to control the effect of confidence-set updates. In particular, the following lemma shows that the confidence-feasible occupancy sets approach the corresponding true-kernel feasible sets as $t\to\infty$. As a consequence, the jumps in the occupancy iterates caused by the asynchronous confidence updates decay sufficiently fast to be summable.

\begin{lemma}\label{lem:restart-jumps}
Let $\operatorname{dist}_1(x,\mathcal C):=\inf_{y\in\mathcal C}\|x-y\|_1$. On the event of Lemma~\ref{lem:asymptotic-visitation},
\begin{align}\label{eq:set-convergence}
\sup_{x\in\Delta_{i,\delta_i}(\mathcal P_i^t)}
\operatorname{dist}_1\!\left(x,\Delta_{i,\delta_i}(P_i)\right)\leq O\big(|\mathcal S_i|
\max_{s_i,a_i}\beta_i^t(s_i,a_i)\big)= O\bigg(|\mathcal S_i|\sqrt{\frac{\log(t+2)}{t+1}}\bigg).
\end{align}
Consequently, if $\mathcal R_i$ denotes the confidence-update times for player $i$, then occupancy jumps satisfy
\begin{align}\label{eq:restart-jump-summable}
\sum_{t\in\mathcal R_i}
\|\hat q_i^{t+1}-\hat q_i^t\|_1<\infty.
\end{align}
\end{lemma}

Finally, by combining the above lemmas we can prove the main result of this section. 

\begin{theorem}\label{theo:assymptotic}
Let Assumptions~\ref{ass:indep}, \ref{ass:ergodic}, and
\ref{ass:g_stable} hold. Assume players follow Algorithm~\ref{alg:main} with $\eta_t=\eta_0(t+1)^{-a}$, $a\in(1/2,1]$. Then, with probability at least $1-\gamma$, $\hat q^t\rightarrow q^*$ and  $q^t\rightarrow q^*$, where $q^*$ is the unique stable NE of the virtual game. Consequently, the induced stationary policies converge to the policy
profile represented by $q^*$. In particular, for sufficiently small action floors $\delta=\max_i \delta_i$, this policy profile is an $\epsilon_\delta$-NE of the original game, where $\epsilon_\delta\to0$ as the action floors $\delta$ tend to zero.
\end{theorem}

\begin{proof}
Work on the anytime confidence event $\mathcal E_{\rm conf}$, which,
by Lemma~\ref{lem:anytime-confidence}, has probability at least
$1-\gamma/4$. On this event, the conclusions of
Lemmas~\ref{lem:asymptotic-visitation} and
\ref{lem:restart-jumps} hold almost surely. In the argument below, we
additionally invoke the concentration events from
Lemma~\ref{lem:delay-coupling} on the dyadic epochs, with failure
probabilities $\{\gamma_m\}_{m\ge0}$ satisfying
$\sum_{m\ge0}\gamma_m\le\gamma/2$. Hence, by a union bound, all the
events used in the proof hold simultaneously with probability at least
$1-3\gamma/4$, and therefore at least $1-\gamma$.

By Assumption~\ref{ass:g_stable}, the unique stable equilibrium
$q^*\in\boldsymbol\Delta_{\boldsymbol\delta}(\boldsymbol P)$ satisfies
\[
G(q):=
\sum_{i=1}^n
\langle g_i(q_{-i}),q_i^*-q_i\rangle
\ge0,
\qquad
G(q)=0\Longleftrightarrow q=q^*.
\]
Define the Lyapunov function
\[
V(\hat q^t):=
\frac12\sum_{i=1}^n
\|q_i^*-\hat q_i^t\|_2^2.
\]
We first show that, for every $m\ge0$, over the dyadic interval $I_m:=\{2^m,\ldots,2^{m+1}-1\}$,
\begin{align}\label{eq:quasi-fejer-rigorous}
V(\hat q^{2^{m+1}})
\le
V(\hat q^{2^m})
-\sum_{t\in I_m}\eta_tG(q^t)
+e_m,
\qquad
e_m\ge0,
\qquad
\sum_{m=0}^{\infty}e_m<\infty.
\end{align}

Since $q_i^*\in\Delta_{i,\delta_i}(\mathcal P_i^t)$ on the anytime
confidence event $\mathcal E_{\rm conf}$, the one-step proximal inequality derived from
\eqref{eq:optimality-euclidean}, with $q_i'=q_i^*$ and stepsize
$\eta_t$, gives
\begin{align}
V(\hat q^{t+1})&-V(\hat q^t)
\le
-\eta_t\sum_i
\langle q_i^*-\hat q_i^t,\widehat R_i^t\rangle
+\frac{\eta_t^2}{2}\sum_i\|\widehat R_i^t\|_2^2 \notag\\
&=
-\eta_tG(q^t)
-\eta_t\sum_i
\langle q_i^*-\hat q_i^t,\widehat R_i^t-g_i^t\rangle
-\eta_t\sum_i
\langle q_i^t-\hat q_i^t,g_i^t\rangle
+\frac{\eta_t^2}{2}\sum_i\|\widehat R_i^t\|_2^2 .
\label{eq:asymptotic-one-step}
\end{align}
Thus, after summing over $t\in I_m$, it remains only to show that the
last three terms in \eqref{eq:asymptotic-one-step}, summed over the
dyadic epochs, are absolutely summable.

For the last two terms in \eqref{eq:asymptotic-one-step}, using
$\|g_i^t\|_\infty\le1$, $\|\widehat R_i^t\|_2\le\sqrt{|\mathcal S_i|}/\delta_i$, and \eqref{eq:asymptotic-q}, we obtain
\begin{align}
\sum_{t=0}^{\infty}
\eta_t
\left|
\sum_i\langle q_i^t-\hat q_i^t,g_i^t\rangle
\right|
&\le
\sum_i\sum_{t=0}^{\infty}
\eta_t\|\hat q_i^t-q_i^t\|_1
\lesssim
\sum_{t=1}^{\infty}
t^{-a}\sqrt{\frac{\log(t+2)}{t}}
<\infty,
\label{eq:asymptotic-q-error-sum}\\
\sum_{t=0}^{\infty}
\frac{\eta_t^2}{2}
\sum_i\|\widehat R_i^t\|_2^2
&\le
\frac12\sum_i\frac{|\mathcal S_i|}{\delta_i^2}
\sum_{t=0}^{\infty}\eta_t^2
<\infty.
\label{eq:asymptotic-quadratic-sum}
\end{align}
It remains to control the estimator error in \eqref{eq:asymptotic-one-step}. For each epoch $I_m$, set $D_m\!=\!
1+\lceil
\frac1\kappa
\log(
\frac{16n\,2^{3(m+1)}}{\gamma_m})
\rceil$ and $\gamma_m:=\frac{3\gamma}{\pi^2(m+1)^2}$. We apply Lemma~\ref{lem:delay-coupling} to each epoch $I_m$, with
$q'=q^*$, horizon $|I_m|=2^m$, delay $D_m=O(m)$, failure probability
$\gamma_m$, and stepsize upper bound
$\bar\eta_m:=\eta_{2^m}\asymp2^{-am}$.\footnote{Although Lemma~\ref{lem:delay-coupling} is stated for a constant
stepsize and the interval $\{0,\ldots,T-1\}$, its proof extends
directly to any interval $I=\{u,\ldots,v-1\}$ and nonincreasing stepsizes $\{\eta_t\}$. Indeed, with
$\bar\eta_I:=\max_{t\in I}\eta_t=\eta_u$, the weights $\eta_t$ can be
carried through the same martingale and bias estimates and bounded by
$\bar\eta_I$, while the policy-movement estimates are controlled
by $\bar\eta_I$. Thus, the same cumulative estimate holds with the
constant stepsize replaced by $\bar\eta_I$.}
Substituting these quantities into the bound of
Lemma~\ref{lem:delay-coupling}, yields, up to
fixed problem-dependent constants and logarithmic factors,
\begin{align}\label{eq:asymptotic-gradient-summability}
&\sum_{m=0}^{\infty}
\left|
\sum_{t\in I_m}
\eta_t
\sum_i
\langle q_i^*-\hat q_i^t,
\widehat R_i^t-g_i^t\rangle
\right|\cr 
&\qquad\lesssim
\sum_{m=0}^{\infty}
\left[
m\,2^{(1/2-a)m}
+m^2\,2^{(1-2a)m}
+\sqrt{m}\,2^{(1/2-a)m}
+m\,2^{-am}
\right]<\infty.
\end{align}
Here, the first three summands in
\eqref{eq:asymptotic-gradient-summability} dominate the terms obtained from
Lemma~\ref{lem:delay-coupling}. The additional last term
$m2^{-am}$ accounts for the first $D_m$ times of each shifted epoch,
whose delay windows extend before the epoch. Bounding these terms directly
gives a total weighted contribution of order
$D_m\bar\eta_m=O(m2^{-am})$. All the sums in \eqref{eq:asymptotic-gradient-summability} are finite because $a>1/2$ and by the choice of $D_m$. Moreover, $\sum_m\gamma_m\leq \frac{\gamma}{2}$, so a union bound over the dyadic epochs makes these estimates simultaneous on the common event under consideration. Finally, summing \eqref{eq:asymptotic-one-step} over $t\in I_m$ and
defining $e_m$ as the sum of the absolute values of the three error
contributions in
\eqref{eq:asymptotic-q-error-sum}--\eqref{eq:asymptotic-gradient-summability}
gives \eqref{eq:quasi-fejer-rigorous}.

We now use \eqref{eq:quasi-fejer-rigorous} to prove convergence.
Since $G(q)\ge0$, summing \eqref{eq:quasi-fejer-rigorous} over $m$
and using the summability of its error terms gives
\begin{align}\label{eq:G-summable}
\sum_{t=1}^{\infty}\eta_t G(q^t)<\infty.
\end{align}
Suppose that $\hat q^t$ has a limit point $\bar q\ne q^*$. By
\eqref{eq:asymptotic-q}, $\|q^t-\hat q^t\|_1\to0$, so $\bar q$ is
also a limit point of $\{q^t\}$. By continuity of $G(\cdot)$ and
Assumption~\ref{ass:g_stable}, there exist $r>0$ and a constant
$c_U>0$ such that $G(q)\ge c_U$ for every
$q\in\boldsymbol\Delta_{\boldsymbol\delta}(\boldsymbol P)$ satisfying
$\|q-\bar q\|_1\le 2r$. Define the neighborhoods
$U_0:=\{q:\|q-\bar q\|_1<r/2\}$ and
$U:=\{q:\|q-\bar q\|_1<r\}$. Since
$\|q^t-\hat q^t\|_1\to0$, for all sufficiently large $t$,
$\hat q^t\in U$ implies $\|q^t-\bar q\|_1\le 2r$, and hence
$G(q^t)\ge c_U$.

Away from confidence-update times, the occupancy-movement estimate
\eqref{eq:q-q+} established in the proof of Lemma~\ref{lem:movement}
implies that
$\|\hat q_i^{t+1}-\hat q_i^t\|_1=O(\eta_t)$. At confidence-update times,
\eqref{eq:restart-jump-summable} and the finiteness of the number of
players imply that $\sum_{i=1}^n\sum_{t\in\mathcal R_i}
\|\hat q_i^{t+1}-\hat q_i^t\|_1<\infty$. In particular, the total magnitude of confidence-update
jumps after time $T$ converges to zero as $T\to\infty$. Since
$\overline{U_0}\subset U$, moving from $U_0$ to outside $U$ requires a
fixed positive amount of total movement. Thus, for every sufficiently
late excursion that starts in $U_0$ and ends upon its first exit from
$U$, the confidence-update jumps can account for only an arbitrarily
small part of this movement, while the remaining movement is bounded
by a constant multiple of the sum of $\eta_t$ over the excursion.
Consequently, there exists a constant $c>0$ such that every
sufficiently late such excursion satisfies
$$\sum_{t\ {\rm during\ the\ excursion}}\eta_t\ge c.$$

Since $\bar q$ is a limit point of $\{\hat q^t\}$, the trajectory
visits $U_0$ infinitely often. It must therefore either eventually
remain in $U$, or make infinitely many excursions from $U_0$ to
outside $U$. In the first case, $G(q^t)\ge c_U$ eventually; since
$\sum_t\eta_t=\infty$, this gives
$\sum_t\eta_tG(q^t)=\infty$. In the second case, during every
sufficiently late excursion, up to its first exit from $U$, we have
$\hat q^t\in U$ and hence $G(q^t)\ge c_U$, while the excursion
satisfies $\sum_{t\ {\rm during\ the\ excursion}}\eta_t\ge c$.
Infinitely many disjoint excursions therefore again give
$\sum_t\eta_tG(q^t)=\infty$. Both cases contradict
\eqref{eq:G-summable}. Thus $q^*$ is the only limit point of
$\{\hat q^t\}$, and compactness yields $\hat q^t\rightarrow q^*$.
Equation~\eqref{eq:asymptotic-q} then immediately gives
$q^t\rightarrow q^*$. 

Finally, the policy map $\pi_i(a_i\mid s_i)
=
\frac{\rho_i(s_i,a_i)}
{\sum_{b_i}\rho_i(s_i,b_i)}$ is continuous on the shrunk occupancy set, since its state marginals
are bounded away from zero. Hence the induced stationary policies
converge to the policy profile represented by $q^*$. Moreover, since
$q^*$ is an NE of the shrunk virtual game, for every player $i$, $V_i(q_i^*,q_{-i}^*)
\ge
\max_{q_i\in\Delta_{i,\delta_i}(P_i)}
V_i(q_i,q_{-i}^*)$. By the same shrinkage argument used in Corollary~\ref{cor:cce}, for
sufficiently small action floors, $\max_{q_i\in\Delta_i(P_i)}
V_i(q_i,q_{-i}^*)
\le
\max_{q_i\in\Delta_{i,\delta_i}(P_i)}
V_i(q_i,q_{-i}^*)+\epsilon_\delta$,
where $\epsilon_\delta\to0$ as $\delta=\max_i\delta_i\to0$.
Combining the last two displays shows that $V_i(q_i^*,q_{-i}^*)
\ge
\max_{q_i\in\Delta_i(P_i)}
V_i(q_i,q_{-i}^*)-\epsilon_\delta,\ \forall i\in[n]$. Thus, the stationary policy profile represented by $q^*$ is an $\epsilon_\delta$-NE of the original game.
\end{proof}

\section{Real-World Applications}\label{sec:applications}

The stochastic game model studied in this paper is relevant to
multi-agent systems in which each player controls the evolution of a
local stochastic state while interacting with other
players through its payoff. This structure arises when the
physical dynamics of individual agents are locally controlled, but
their objectives are coupled through a shared resource, market, or
environment. We illustrate this connection through three representative
applications: energy management in smart grids, resource allocation in
wireless networks, and decentralized robotic navigation and coverage.
For these applications, the framework developed in this paper provides
a fully decentralized way for agents to learn from a single continuing
trajectory without knowing their local transition kernels or observing
the states and actions of other players, thereby respecting users'
privacy and accommodating communication or information bottlenecks
that hinder coordination. Our results provide finite-time no-regret and
asymptotic convergence guarantees whose complexity is governed by local
state spaces and cover times, making them scalable and practical for
large-scale systems. 

\subsection{Energy Management in Smart Grids}

Consider a collection of $n$ geographically distributed energy users,
indexed by $[n]=\{1,\ldots,n\}$, interacting with a utility company.
Each user acts as a \emph{prosumer}: it can consume electricity from
the grid while also producing energy locally through renewable sources
such as solar panels or wind turbines. In addition, user $i$ is
equipped with a storage device of capacity $C$, allowing unused energy
to be carried from one period to the next. Let $s_i^t$ denote the quantized energy stored by user $i$ at the
beginning of day $t$, and let $h_i^t$ denote the renewable energy
harvested during that day. The random variable $h_i^t$ captures
uncertainty in local renewable production due, for example, to
weather conditions. If $a_i^t$ denotes the amount of energy consumed
by user $i$ during day $t$, its storage level evolves according to $s_i^{t+1}=\min\big\{C,\;h_i^t+\left(s_i^t-a_i^t\right)^+\big\}$,
where $x^+:=\max\{x,0\}$. Thus, the current consumption decision
directly influences the distribution of the user's future storage
state. When $a_i^t>s_i^t$, the user must obtain $\left(a_i^t-s_i^t\right)^+$ units of electricity from the grid to meet its consumption requirement. Although the storage dynamics are local, the users interact
strategically through the electricity market. In particular, suppose
that the utility determines the unit price of grid energy based on
the aggregate demand, and denote this price by $p(a^t,s^t)$. If
$u_i(a_i^t)$ denotes the benefit that user $i$ obtains from consuming
$a_i^t$ units of energy, its instantaneous reward can be written as $r_i(a^t,s^t)=u_i(a_i^t)-p(a^t,s^t)\left(a_i^t-s_i^t\right)^+$. Hence, while player $i$ controls only its own storage dynamics, its
reward depends on the states and consumption decisions of the entire
population through the market price (see Figure~\ref{Fig:applications}--a).

This model fits the independent-chain framework particularly well when
the users are sufficiently geographically separated that their local
renewable-generation processes can be modeled as independent.
Conditional on $(s_i^t,a_i^t)$, the next storage state of
player $i$ is then governed by its own renewable-generation process
and does not depend on the states or actions of other users.
Nevertheless, the players remain strategically coupled because the
price $p(a^t,s^t)$ depends on their collective demand. Consequently,
each user seeks to learn a long-run consumption and storage policy
from its own observations without knowing the stochastic model
governing its renewable generation or observing the states and actions
of other users. This provides a decentralized energy-management
application of the stochastic game studied in this paper
\citep{etesami2018stochastic}.

\begin{figure}[!tbp]
    \centering
    \includegraphics[width=\textwidth]{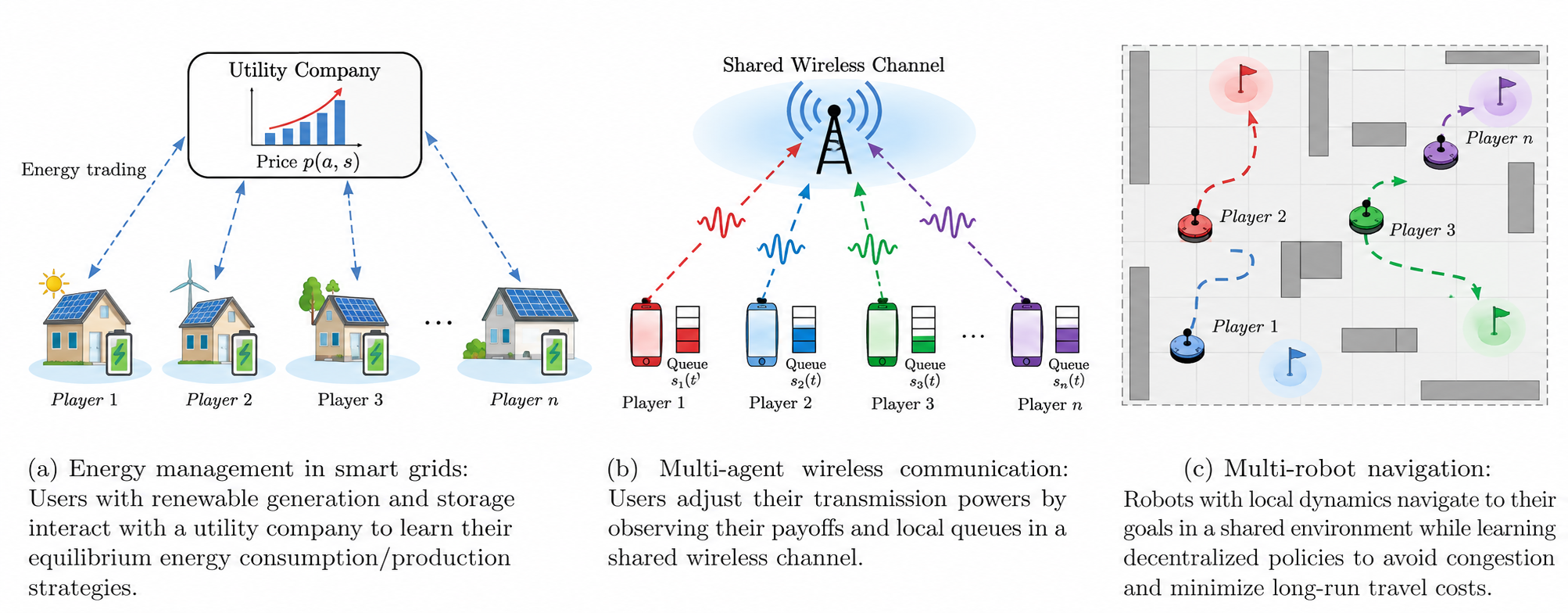}
    \caption{\footnotesize Representative applications of stochastic
    games with independent controlled chains. (a) Energy management:
    prosumers with local renewable generation and storage interact
    through a utility company and an aggregate-demand-dependent price.
    (b) Wireless communication: users with independent local queues
    adjust transmission decisions while competing over a shared
    channel. (c) Robotic navigation and coverage: robots with local
    controlled dynamics navigate through a shared environment while
    balancing individual navigation objectives with collective
    criteria such as congestion avoidance and spatial coverage.}
    \label{Fig:applications}
\end{figure}

\subsection{Power and Bandwidth Allocation in Wireless Networks}

A second example arises in decentralized wireless communication
systems in which several transmitters compete for access to a common
channel \citep{altman2007constrained}. Consider $n$ users transmitting packets to one or
more receivers over a shared wireless medium. For transmitter $i$,
let $s_i^t$ represent its local queue state at time $t$, describing,
for example, the number of packets waiting to be transmitted. New
packets arrive according to a local stochastic arrival process that is
independent of the packet-generation processes of the other users. At each time $t$, transmitter $i$ observes its own queue and selects
an action $a_i^t$, which may specify a transmission-power level,
bandwidth allocation, transmission rate, or channel-access decision.
Its queue then evolves according to the packets that arrive locally
and the packets served as a consequence of its transmission decision.
Thus, the action of each transmitter controls the evolution of its
own queue state, while the underlying packet-arrival uncertainty is
local to that user. The strategic interaction occurs through the shared communication
medium. For example, the probability that player $i$ successfully
transmits a packet may depend on the transmission powers of all users,
since simultaneous transmissions generate interference. More
generally, the reward $r_i(s^t,a^t)$ may account for successful throughput, transmission energy, queueing
delay, packet loss, or congestion, and may therefore depend on the
joint state--action profile. Channel conditions, such as background
noise and available capacity, may further influence the realized
payoff (see Figure~\ref{Fig:applications}--b).

The resulting system separates local dynamics from strategic
interaction in precisely the manner captured by the independent-chain
model: each user's queue evolves according to its own controlled local
dynamics, whereas the performance resulting from a transmission
decision depends on the simultaneous behavior of the other users.
A transmitter can therefore seek to learn a stationary power- or
bandwidth-allocation policy using only its local queue, selected
transmission action, observed next queue state, and realized
transmission payoff, without knowing its local transition law or
observing the states and actions of the competing transmitters.

\subsection{Decentralized Robotic Navigation and Coverage}

A third application arises in decentralized robotic navigation and
coverage, where a collection of autonomous agents moves through a common
environment while collectively monitoring, exploring, or covering a region
of interest. Such cooperative-control problems have been studied using
potential, state-based potential, and distributed welfare games
\citep{marden2009cooperative,marden2012state,marden2013distributed}.
To connect this setting to the independent-chain model, suppose that
robot $i$ has a quantized local state $s_i^t$ describing its position,
heading, velocity, battery level, sensing status, or other locally
available information relevant to its motion and operation. At time $t$,
the robot observes its current local state and selects a local control
action $a_i^t$, such as a movement direction, velocity adjustment, or
sensing decision. Its next state is then generated according to an
unknown controlled transition kernel
$P_i(s_i^{t+1}\mid s_i^t,a_i^t)$, which may capture uncertainty in
actuation, localization, terrain conditions, sensor measurements, or
other robot-specific disturbances encountered during operation
(see Figure~\ref{Fig:applications}--c). When these disturbances are
independent across robots, each robot's state evolves according to its
own locally controlled stochastic dynamics, and the resulting system
has the independent controlled-chain structure considered in this paper.

The robots nevertheless interact strategically through
their rewards, which may account for travel time, energy consumption,
congestion, collision avoidance, or progress toward a destination.
The reward may also incorporate a collective \emph{coverage} objective.
For example, the workspace can be partitioned into Voronoi cells, with
each robot responsible for monitoring its associated region, and the
collective reward can favor configurations that cover a larger fraction
of the workspace, reduce the aggregate distance between points and
their nearest robots, or place robots near the centroids of their
Voronoi regions
\citep{marden2009cooperative,marden2013distributed}. Accordingly, the
reward $r_i(s^t,a^t)$ can combine individual navigation and energy
costs with team-dependent criteria such as coverage, congestion, and
collision avoidance. Thus, although each robot controls only its local
dynamics, its reward can depend on the configuration and actions of the
entire robot population. The resulting equilibrium policies can
therefore represent decentralized motion strategies that balance local
navigation costs with collective objectives such as efficient spatial
coverage.

\section{Conclusion}\label{sec:conclusion}

We developed a fully online one-sample occupancy-measure learning algorithm for stochastic games with unknown independent controlled chains. The algorithm operates along a single continuing trajectory, using one local transition and realized payoff per time step, while players update their policies and learn their unknown transition kernels asynchronously using only local information. Under uniform ergodicity, each individual controlled chain has a finite worst-case expected cover time $H_i$, which enters the finite-time complexity as a local problem-dependent parameter. A frozen-policy block comparison, together with the slow movement of mirror descent, transfers fixed-policy coverage and mixing properties to the time-inhomogeneous online trajectory. The resulting high-probability fixed-comparator regret decreases at the rate $\widetilde{\mathcal O}(T^{-1/2})$ after a local-cover transient period. For arbitrary reward functions, this regret guarantee further implies that the empirical distribution of play approaches a coarse-correlated equilibrium. Because transition estimation and delayed moving-chain comparisons are performed separately on each local state space $\mathcal S_i$, neither the cover time nor the cardinality of the product state space $\prod_i\mathcal S_i$ enters the complexity. Under an additional variational-stability condition and diminishing stepsizes, the same fully online scheme further yields asymptotic last-iterate convergence to a stationary $\epsilon$-NE.

\bibliography{thesisrefs}

\paragraph{Acknowledgment.}
The core ideas and technical contributions of this work originated with the author and were inspired by his earlier collaborative work on episodic learning \citep{qin2024learning}, together with his recent work on developing the first fully online algorithm for Markov $\alpha$-potential games \citep{etesami2026markovpotential}. The author acknowledges the use of generative AI (GPT-5.6 Sol) for brainstorming and for technical and writing assistance, including refining and checking proofs, filling in standard technical gaps, improving exposition and grammar, and preparing figures. All AI-assisted material was independently verified and revised by the author.

\appendix
\renewcommand{\thesection}{\Alph{section}}
\renewcommand{\thesubsection}{\thesection.\arabic{subsection}}
\setcounter{section}{0}

\section{Omitted Proofs for Section \ref{sec:dual} }\label{subsec:appendixa}

\subsection{Proof of Lemma \ref{prop:dual_formulation}}

\begin{proof} 
It has been shown in \citep{etesami2022learning} that, under Assumptions~\ref{ass:indep} and~\ref{ass:ergodic}, if each player $i$ follows a stationary policy $\pi_i$, then
\begin{align}\label{eq:linear-occupation-form1}
    V_i(\pi_i,\pi_{-i})
    =V_i(\rho_i,\rho_{-i})
    =\sum_{s,a}\prod_{j=1}^n\rho_j(s_j,a_j)r_i(s,a)
    =\langle \rho_i,v_i(\rho_{-i})\rangle,
\end{align}
where $v_i(\rho_{-i})$ denotes player $i$'s payoff-gradient vector with respect to $\rho_i$; it is a vector of dimension $|\mathcal S_i||\mathcal A_i|$ whose $(s_i,a_i)$-th coordinate is given by
\begin{align}\label{eq:v-rho-coordinate}
    v_i(\rho_{-i})_{(s_i,a_i)}
    =\sum_{s_{-i},a_{-i}}
    \prod_{j\neq i}\rho_j(s_j,a_j)r_i(s,a).
\end{align}
Since
$\rho_j(s_j,a_j)=\sum_{s'_j}q_j(s_j,a_j,s'_j)$
for every player $j$, substituting this relation into \eqref{eq:linear-occupation-form1} gives
\begin{align*}
V_i(\rho_i,\rho_{-i})
&=
\sum_{s_i,a_i,s'_i}
q_i(s_i,a_i,s'_i)
\sum_{s_{-i},a_{-i}}
\prod_{j\neq i}
\sum_{s'_j}q_j(s_j,a_j,s'_j)r_i(s,a)=\langle q_i,g_i(q_{-i})\rangle,
\end{align*}
which establishes \eqref{eq:linear-occupation-form2} and \eqref{eq:l_def}. Now, comparing \eqref{eq:l_def} and \eqref{eq:v-rho-coordinate}, and using $\rho_j(s_j,a_j)=\sum_{s'_j}q_j(s_j,a_j,s'_j)$ shows that
\[
g_i(q_{-i})_{(s_i,a_i,s'_i)}
=
v_i(\rho_{-i})_{(s_i,a_i)},
\qquad
\forall\,s_i,a_i,s'_i.
\]
In particular, $g_i(q_{-i})_{(s_i,a_i,s'_i)}$ is independent of $s'_i$, with the same value repeated across all $s'_i\in\mathcal S_i$.
\end{proof}

\subsection{Proof of Lemma \ref{lem:confidence-convex}}

\begin{proof}
For $q_i\in\Delta_{i,\delta_i}$, since $\rho_i(s_i,a_i)\ge\delta_i>0$ where we note that $\rho_i(s_i,a_i)=\sum_{y_i}q_i(s_i,a_i,y_i)$, the transition kernel induced by
$q_i$ is well defined and satisfies
\[
P_i^{q_i}(s_i'\mid s_i,a_i)
=
\frac{q_i(s_i,a_i,s_i')}{\rho_i(s_i,a_i)}.
\]
Therefore, the condition
$P^{q_i}\in\mathcal P_i$ is equivalent to imposing the box constraints $L_i(s_i,a_i,s_i')
\le
P_i^{q_i}(s_i'\mid s_i,a_i)
\le
U_i(s_i,a_i,s_i')$, which are equivalent to the linear inequalities
\[
L_i(s_i,a_i,s_i')\sum_{y_i}q_i(s_i,a_i,y_i)
\le q_i(s_i,a_i,s_i')
\le
U_i(s_i,a_i,s_i')\sum_{y_i}q_i(s_i,a_i,y_i).
\]
Consequently, $\Delta_{i,\delta_i}(\mathcal P_i)$ admits the following explicit representation as the set of all occupancies $q_i$ satisfying the following linear constraints:
\[
\begin{array}{ll}
q_i(s_i,a_i,s_i')\ge0,
& \forall s_i,a_i,s_i',\\[1mm]
\displaystyle
\sum_{s_i,a_i,s_i'}q_i(s_i,a_i,s_i')=1,
& \\[2mm]
\displaystyle
\sum_{s_i',a_i}q_i(s_i',a_i,s_i)
=
\sum_{s_i',a_i}q_i(s_i,a_i,s_i'),
& \forall s_i,\\[2mm]
\displaystyle
\sum_{s_i'}q_i(s_i,a_i,s_i')\ge\delta_i,
& \forall s_i,a_i,\\[2mm]
\displaystyle
L_i(s_i,a_i,s_i')\sum_{y_i}q_i(s_i,a_i,y_i)
\!\le\! q_i(s_i,a_i,s_i')
\!\le\!
U_i(s_i,a_i,s_i')\sum_{y_i}q_i(s_i,a_i,y_i),
& \forall s_i,a_i,s_i'.
\end{array}
\]
All the displayed constraints in $\Delta_{i,\delta_i}(\mathcal P_i)$ are linear in $q_i$ with
$|\mathcal S_i|^2|\mathcal A_i|$ variables and
$O(|\mathcal S_i|^2|\mathcal A_i|)$ constraints. Moreover, the set is
closed and bounded by nonnegativity and normalization. Hence
$\Delta_{i,\delta_i}(\mathcal P_i)$ is a closed convex polytope with an
explicit polynomial-size linear description.
\end{proof}

\section{Omitted Proofs for Section \ref{sec:alg-design}}

\subsection{Proof of Lemma \ref{lem:anytime-confidence}}

\begin{proof}
Fix $i\in[n]$, $(s_i,a_i)\in\mathcal S_i\times\mathcal A_i$, and
$s_i'\in\mathcal S_i$. All random variables below are defined on the
common probability space supporting the entire game trajectory. For
each $\ell\ge1$, let $S_{i,\ell}^+$ denote the next-state observation
following the $\ell$th occurrence of the state--action pair
$(s_i,a_i)$. Let $\mathcal F_{i,0}$ contain the entire game history
immediately after the first occurrence of $(s_i,a_i)$ and before
observing $S_{i,1}^+$. More generally, for $\ell\ge1$, let
$\mathcal F_{i,\ell}$ contain the entire game history immediately
after the $(\ell+1)$st occurrence of $(s_i,a_i)$ and before observing
$S_{i,\ell+1}^+$. Thus,
$\{\mathcal F_{i,\ell}\}_{\ell\ge0}$ is the filtration obtained by
indexing the entire game history by successive occurrences of
$(s_i,a_i)$. In particular, $\mathcal F_{i,\ell-1}$ contains the
entire history preceding $S_{i,\ell}^+$, including the fact that the
$\ell$th occurrence of $(s_i,a_i)$ has taken place, but does not
contain $S_{i,\ell}^+$ itself. For $m\ge1$, define
\[
\bar P_{i,m}(s_i'\mid s_i,a_i)
:=
\frac1m\sum_{\ell=1}^m
{\bf1}\{S_{i,\ell}^+=s_i'\}.
\]
By the controlled Markov property, conditional on
$\mathcal F_{i,\ell-1}$, the next state following the $\ell$th
occurrence of $(s_i,a_i)$ is distributed according to
$P_i(\cdot\mid s_i,a_i)$. Therefore,
\[
\mathbb E\left[
{\bf1}\{S_{i,\ell}^+=s_i'\}
\,\middle|\,
\mathcal F_{i,\ell-1}
\right]
=
P_i(s_i'\mid s_i,a_i).
\]
Consequently,
\[
X_{i,\ell}
:=
{\bf1}\{S_{i,\ell}^+=s_i'\}
-
P_i(s_i'\mid s_i,a_i),
\qquad \ell\ge1,
\]
is a martingale-difference sequence with respect to
$\{\mathcal F_{i,\ell}\}_{\ell\ge0}$. Moreover,
$X_{i,\ell}$ takes values in an interval of length one. Hence, by
Hoeffding--Azuma, for every $m\ge1$ and $x>0$,
\[
\mathbb P\left(
\left|
\bar P_{i,m}(s_i'\mid s_i,a_i)
-
P_i(s_i'\mid s_i,a_i)
\right|
>
\sqrt{\frac{2x}{m}}
\right)
\le 2e^{-x}.
\]

Now consider any time $t$ such that
$N_i^t(s_i,a_i)=m\ge1$. By the definitions of $N_i^t$,
$\bar P_i^t$, and $\bar P_{i,m}$,
\[
\bar P_i^t(s_i'\mid s_i,a_i)
=
\bar P_{i,m}(s_i'\mid s_i,a_i).
\]
Moreover, $m\le t$, and $\Lambda_t$ is nondecreasing, so
\[
\Lambda_t\ge\Lambda_m
=
\log\left(
\frac{64nA_{\max}S_{\max}^2(m+2)^2}{\gamma}
\right).
\]
Therefore, if for some $t$ with $N_i^t(s_i,a_i)=m\ge1$,
\[
\left|
\bar P_i^t(s_i'\mid s_i,a_i)
-
P_i(s_i'\mid s_i,a_i)
\right|
>
\beta_i^t(s_i,a_i),
\]
then necessarily
\[
\left|
\bar P_{i,m}(s_i'\mid s_i,a_i)
-
P_i(s_i'\mid s_i,a_i)
\right|
>
\sqrt{\frac{2\Lambda_m}{m}}.
\]
Thus, taking a union bound over all possible visit counts $m\ge1$, \footnote{If $N_i^t(s_i,a_i)=0$, then
$\bar P_i^t(s_i'\mid s_i,a_i)=0$, while
$P_i(s_i'\mid s_i,a_i)\le1$. Moreover,
$\beta_i^t(s_i,a_i)=\sqrt{2\Lambda_t}\ge1$. Hence, $\left|
\bar P_i^t(s_i'\mid s_i,a_i)
-
P_i(s_i'\mid s_i,a_i)
\right|
\le
\beta_i^t(s_i,a_i)$,
so the confidence bound holds automatically.}
\begin{align*}
&\mathbb P\left(
\exists\,t\ge1:
N_i^t(s_i,a_i)\ge1,\ 
\left|
\bar P_i^t(s_i'\mid s_i,a_i)
-
P_i(s_i'\mid s_i,a_i)
\right|
>
\beta_i^t(s_i,a_i)
\right)
\\
&\qquad\le
\sum_{m=1}^{\infty}
\mathbb P\left(
\left|
\bar P_{i,m}(s_i'\mid s_i,a_i)
-
P_i(s_i'\mid s_i,a_i)
\right|
>
\sqrt{\frac{2\Lambda_m}{m}}
\right)\le
2\sum_{m=1}^{\infty}e^{-\Lambda_m}
\\
&\qquad=
2\sum_{m=1}^{\infty}
\left(
\frac{\gamma}
{64nA_{\max}S_{\max}^2(m+2)^2}
\right).
\end{align*}

By taking a union bound over all $i,s_i,a_i,s_i'$ and using $\sum_{i=1}^n
|\mathcal A_i||\mathcal S_i|^2
\le
nA_{\max}S_{\max}^2$, we obtain
\begin{align*}
&\mathbb P\left(
\exists\,i,t,s_i,a_i,s_i':
\left|
\bar P_i^t(s_i'\mid s_i,a_i)
-
P_i(s_i'\mid s_i,a_i)
\right|
>
\beta_i^t(s_i,a_i)
\right)
\\
&\qquad\le
2nA_{\max}S_{\max}^2
\sum_{m=1}^{\infty}
\left(
\frac{\gamma}
{64nA_{\max}S_{\max}^2(m+2)^2}
\right)
<\frac{\gamma}{4}.
\end{align*}
Thus, with probability at least $1-\gamma/4$, every confidence region
constructed by every player at every confidence-update time contains
the corresponding true transition kernel $P_i$. Since
$\mathcal P_i^0$ contains all transition kernels and
Algorithm~\ref{alg:main} updates $\mathcal P_i^t$ only by intersecting
the previous confidence set with such confidence regions, it follows
inductively that
\[
P_i\in\mathcal P_i^t,
\qquad
\forall i\in[n],\ \forall t\ge0.
\]
Hence, $\mathbb P(\mathcal E_{\rm conf})\ge1-\gamma/4$. Finally, on $\mathcal E_{\rm conf}$, if
$q_i\in\Delta_{i,\delta_i}(P_i)$, then its induced transition kernel
satisfies $P^{q_i}=P_i\in\mathcal P_i^t$. Therefore, $q_i\in\Delta_{i,\delta_i}(\mathcal P_i^t)$, and consequently
\[
\Delta_{i,\delta_i}(P_i)
\subseteq
\Delta_{i,\delta_i}(\mathcal P_i^t),
\qquad
\forall i\in[n],\ \forall t\ge0.
\]
\end{proof}

\section{Omitted Proofs for Section \ref{sec:finite_analysis}}

\subsection{Proof of Lemma \ref{lem:finite-md}}

\begin{proof}
On the event $\mathcal E_{\rm conf}$, $P_i\in\mathcal P_i^t$ for every $t$, hence the fixed comparator $q_i'$ belongs to every feasible set $\Delta_{i,\delta_i}(\mathcal P_i^t)$. The first-order optimality condition for the proximal occupancy update in Algorithm~\ref{alg:main} therefore gives
\begin{align}\label{eq:optimality-euclidean}
\left\langle\hat q_i^{t+1}-\hat q_i^t-\eta\widehat R_i^t,q_i'-\hat q_i^{t+1}\right\rangle\ge0.
\end{align}
Now, we can write
\begin{align}\label{eq:optimality-quadratic-VI}
\eta\langle q_i'-\hat q_i^t,\widehat R_i^t\rangle
&=
\eta\langle q_i'-\hat q_i^{t+1},\widehat R_i^t\rangle
+\eta\langle\hat q_i^{t+1}-\hat q_i^t,\widehat R_i^t\rangle\cr
&\le
\langle\hat q_i^{t+1}-\hat q_i^t,
q_i'-\hat q_i^{t+1}\rangle
+\eta\langle\hat q_i^{t+1}-\hat q_i^t,\widehat R_i^t\rangle\cr
&=
\frac12\|q_i'-\hat q_i^t\|_2^2
-\frac12\|q_i'-\hat q_i^{t+1}\|_2^2
-\frac12\|\hat q_i^{t+1}-\hat q_i^t\|_2^2+\eta\langle\hat q_i^{t+1}-\hat q_i^t,\widehat R_i^t\rangle\cr
&\le
\frac12\|q_i'-\hat q_i^t\|_2^2
-\frac12\|q_i'-\hat q_i^{t+1}\|_2^2
+\frac{\eta^2}{2}\|\widehat R_i^t\|_2^2,
\end{align}
where the first inequality follows from the first-order optimality condition \eqref{eq:optimality-euclidean}, the second equality uses the three-point identity $2\langle x-y,z-x\rangle
=\|z-y\|_2^2-\|z-x\|_2^2-\|x-y\|_2^2$, and the last inequality follows from Young's inequality, which cancels the term
$-\frac12\|\hat q_i^{t+1}-\hat q_i^t\|_2^2$.

Summing the above relation over $t$ telescopes globally. Importantly, a confidence update does not create a restart term in this telescope because the true-kernel comparator remains feasible after every confidence-set intersection. Since both occupancy vectors have unit mass, $\|q_i'-\hat q_i^0\|_2^2\le2$. Moreover, by~\eqref{eq:online-estimator}, $\widehat R_i^t$ has exactly $|\mathcal S_i|$ possibly nonzero next-state coordinates, all corresponding to the observed pair $(S_i^t,A_i^t)$, and every such coordinate is at most $1/\delta_i$. Hence $\|\widehat R_i^t\|_2^2\le\frac{|\mathcal S_i|}{\delta_i^2}$. Substitution proves~\eqref{eq:online-md-regret}.
\end{proof}

\subsection{Proof of Lemma \ref{lem:movement}}

\begin{proof}
We first establish the policy-movement bound. Suppose that player $i$ does not update its confidence set at step $t$. Then the feasible set
in the proximal step is the same set that contains $\hat q_i^t$.
Hence the Euclidean update is the projection of
$\hat q_i^t+\eta\widehat R_i^t$ onto this fixed convex set, and
nonexpansiveness of projection gives
\[
\|\hat q_i^{t+1}-\hat q_i^t\|_2
\le
\eta\|\widehat R_i^t\|_2
\le
\frac{\eta\sqrt{|\mathcal S_i|}}{\delta_i}.
\]
Since the occupancy vector has dimension
$|\mathcal A_i||\mathcal S_i|^2$,
\begin{align}\label{eq:q-q+}
\|\hat q_i^{t+1}-\hat q_i^t\|_1
\le
|\mathcal S_i|\sqrt{|\mathcal A_i|}
\,\|\hat q_i^{t+1}-\hat q_i^t\|_2
\le
\frac{L_i\eta}{\delta_i}.
\end{align}
By triangle inequality, marginalization cannot increase the $\ell_1$ norm, so
\[
\|\hat\rho_i^{t+1}-\hat\rho_i^t\|_1
\le
\frac{L_i\eta}{\delta_i}.
\]
For a fixed $s_i$, since every state--action marginal is at least $\delta_i$, $\hat{\nu}^t_i(s_i)=\sum_{a_i}\hat\rho_i^t(s_i,a_i)\ge |\mathcal A_i|\delta_i\ge\delta_i$ and similarly $\hat{\nu}^{t+1}_i(s_i)\ge \delta_i$. Therefore,
\[
\begin{aligned}
\|\pi_i^{t}(\cdot\mid s_i)-\pi_i^{t+1}(\cdot\mid s_i)\|_1=\left\|
\frac{\hat\rho_i^t(s_i,\cdot)}
{\hat{\nu}^{t}(s_i)}
-
\frac{\hat\rho_i^{t+1}(s_i,\cdot)}
{\hat{\nu}^{t+1}(s_i)}
\right\|_1\le
\frac{2\|\hat\rho_i^{t+1}-\hat\rho_i^t\|_1}{\delta_i}
\le
\frac{2L_i\eta}{\delta_i^2}.
\end{aligned}
\]
This proves~\eqref{eq:pi-movement}. In particular, since the controlled kernel is affine in the policy, for every such step,
\begin{align}\label{eq:kernel-movement}
\|P_i^{\pi_i^{t+1}}(\cdot\mid s_i)
-P_i^{\pi_i^t}(\cdot\mid s_i)\|_1
&=
\Big\|
\sum_{a_i}
\bigl(\pi_i^{t+1}(a_i\mid s_i)-\pi_i^t(a_i\mid s_i)\bigr)
P_i(\cdot\mid s_i,a_i)
\Big\|_1\cr
&\le
\sum_{a_i}
\left|\pi_i^{t+1}(a_i\mid s_i)-\pi_i^t(a_i\mid s_i)\right|\cr 
&=
\|\pi_i^{t+1}(\cdot\mid s_i)-\pi_i^t(\cdot\mid s_i)\|_1\leq \frac{2L_i\eta}{\delta_i^2} \ \ \ \ \ \forall s_i\in \mathcal{S}_i.
\end{align}

We now prove the occupancy-estimation bound. The argument proceeds in
three steps. First, we show that every local state--action pair is
visited sufficiently often despite the slowly changing policy. This
controls the confidence radii. Second, we use the confidence sets to
bound the error between the represented and true transition kernels.
Finally, the contraction assumption converts the kernel error into an
occupancy error.

\paragraph{Step 1: Sufficient visitation.}
We first show that the slowly varying policy does not destroy the
local-cover property. Fix player $i$ and an interval $[u,v]$ containing
no confidence update of player $i$. Conditional on $\mathcal F_u$,
consider two local state processes starting from the same state at
time $u$: $S_i^u=\bar S_i^u$. The process $(S_i^t)_{t=u}^v$ is the actual local state trajectory and, at
time $t$, uses the current controlled kernel $P_i^{\pi_i^t}$. In contrast, $(\bar S_i^t)_{t=u}^v$ is the frozen
trajectory: throughout the entire interval it uses the single kernel
$P_i^{\pi_i^u}$. Thus, $S_i^{t+1}\sim P_i^{\pi_i^t}(\cdot\mid S_i^t)$ and $\bar S_i^{t+1}\sim P_i^{\pi_i^u}(\cdot\mid \bar S_i^t)$.

Couple the two chains maximally at each step as long as they have not
separated. Conditional on $S_i^t=\bar S_i^t=s_i$, by maximal coupling their probability
of separating at the next transition is at most
\[
\frac12
\Big\|P_i^{\pi_i^t}(\cdot\mid s_i)
-P_i^{\pi_i^u}(\cdot\mid s_i)\Big\|_1.
\]
Therefore, by a union bound over the transitions, we obtain
\begin{align}\label{eq:local-path-coupling}
\mathbb{P}\Big(
\exists\, t\in[u, v]\text{ such that }S_i^t\neq\bar S_i^t
\,\Big|\,\mathcal F_u
\Big)
\le
\frac12\sum_{t=u}^{v-1}
\max_{s_i}
\Big\|P_i^{\pi_i^t}(\cdot\mid s_i)
-P_i^{\pi_i^u}(\cdot\mid s_i)\Big\|_1.
\end{align}
As long as no confidence update occurs, \eqref{eq:kernel-movement}
holds at every step. Hence, for $t\ge u$,
\[
\max_{s_i}
\Big\|P_i^{\pi_i^t}(\cdot\mid s_i)
-P_i^{\pi_i^u}(\cdot\mid s_i)\Big\|_1
\le
\sum_{k=u}^{t-1}
\max_{s_i}
\Big\|P_i^{\pi_i^{k+1}}(\cdot\mid s_i)
-P_i^{\pi_i^k}(\cdot\mid s_i)\Big\|_1
\le
\frac{2L_i(t-u)\eta}{\delta_i^2}.
\]
Substituting into \eqref{eq:local-path-coupling} gives
\begin{align}\label{eq:local-path-coupling-movement}
\mathbb{P}\Big(
\exists\, t\in[u, v]\text{ such that }S_i^t\neq\bar S_i^t
\,\Big|\,\mathcal F_u
\Big)
\le
\frac{L_i(v-u)^2\eta}{\delta_i^2}.
\end{align}

Next, we apply this comparison on blocks of length $2H_i+1$. Let
\[
\ell_m:=m(2H_i+1),
\qquad
B_m:=\{\ell_m,\ldots,\ell_m+2H_i\},
\]
and let $\tau_m$ be the first confidence-update time of player $i$ at
or after $\ell_m$. Thus, $\tau_m\in B_m$ if and only if a confidence
update occurs in $B_m$. Couple the actual chain with the chain frozen at $\pi_i^{\ell_m}$
until the first confidence update or the end of $B_m$, whichever comes
first. On the event $\{\tau_m\notin B_m\}$, this comparison therefore
runs throughout the entire block. By the same maximal-coupling argument
as in \eqref{eq:local-path-coupling-movement} and
\eqref{eq:T-cover-threshold}, we obtain
\begin{align}\label{eq:coupling-bound}
\mathbb{P}\Big(
\{\tau_m\notin B_m\}
\cap
\{\exists\,t\in B_m\text{ such that }S_i^t\neq\bar S_i^t\}
\,\Big|\,\mathcal F_{\ell_m}
\Big)
\le
\frac{4L_iH_i^2\eta}{\delta_i^2}
\le\frac14.
\end{align}

Conditional on $\mathcal F_{\ell_m}$, the frozen process starts from
$S_i^{\ell_m}$ and evolves according to the homogeneous Markov chain
$P_i^{\pi_i^{\ell_m}}$ induced by the fixed policy $\pi_i^{\ell_m}$.
Let $\tau_{{\rm cov},i}^{\pi_i^{\ell_m}}$ denote its cover time
starting from $S_i^{\ell_m}$. By the definition of $H_i$ and Markov's
inequality,
\[
\mathbb{P}\!\left(
\tau_{{\rm cov},i}^{\pi_i^{\ell_m}}\le 2H_i
\,\middle|\,\mathcal F_{\ell_m}
\right)
=
1-
\mathbb{P}\!\left(
\tau_{{\rm cov},i}^{\pi_i^{\ell_m}}>2H_i
\,\middle|\,\mathcal F_{\ell_m}
\right)
\ge
1-\frac{
\E[\tau_{{\rm cov},i}^{\pi_i^{\ell_m}}
\mid\mathcal F_{\ell_m}]
}{2H_i}
\ge\frac12.
\]
Thus, conditional on $\mathcal F_{\ell_m}$, the frozen chain visits
every local state, and hence any prescribed $s_i$, during $B_m$ with
probability at least $1/2$. Combining this with the coupling bound
\eqref{eq:coupling-bound}, the event that either a confidence update
occurs in $B_m$ or the actual chain visits $s_i$ has conditional
probability at least $1/2-1/4=1/4$. Therefore, for every
$s_i\in\mathcal S_i$,
\begin{align}\label{eq:hit-or-update}
\mathbb{P}\!\Big(
\{\exists\,t\in B_m:S_i^t=s_i\}
\cup
\{\tau_m\in B_m\}
\,\Big|\,\mathcal F_{\ell_m}
\Big)
\ge\frac14.
\end{align}

We now lower bound the visit count $N_i^t(s_i,a_i)$. Fix a prescribed
$(s_i,a_i)$. For each block, we count it as successful if it either
produces an observation of $(s_i,a_i)$ or contains a confidence update;
the latter blocks will later be charged to the total number of confidence
updates. Accordingly, define
\[
Y_m^i(s_i,a_i)
:=
{\bf1}\Big\{
\text{$(s_i,a_i)$ is observed in $B_m$, or $\tau_m\in B_m$}
\Big\}.
\]
Thus, $Y_m^i(s_i,a_i)$ is a ``visit-or-update'' indicator: every block
with $Y_m^i(s_i,a_i)=1$ either contributes at least one visit to
$(s_i,a_i)$ or can be charged to a confidence update. More precisely, fix any $t\le T$ and let $M_t:=\lfloor\frac{t}{2H_i+1}\rfloor$ be the number of complete blocks before time $t$, and $J_i(t)$ be the number of confidence updates before time $t$. Since the blocks are disjoint, for these $M_t$ complete blocks,
\begin{align}\label{eq:Y-to-count}
N_i^{t}(s_i,a_i)
\ge
\sum_{m=0}^{M_t-1}Y_m^i(s_i,a_i)-J_i(t)\geq \sum_{m=0}^{M_t-1}Y_m^i(s_i,a_i)-|\mathcal S_i||\mathcal A_i|
\bigl(1+\lfloor\log_2 t\rfloor\bigr),
\end{align}
where the second inequality follows by the dyadic update rule for confidence updates.

It remains to show that sufficiently many blocks have
$Y_m^i(s_i,a_i)=1$. By~\eqref{eq:hit-or-update} and the fact that the action floor~\eqref{eq:action-floor} gives conditional probability at least $\delta_i$ of selecting $a_i$, we have
\[
\E\!\left[
Y_m^i(s_i,a_i)
\,\middle|\,
\mathcal F_{\ell_m}
\right]
\ge\frac{\delta_i}{4}.
\]

Applying Freedman's inequality\footnote{We use the standard Freedman
inequality for martingale differences $X_m$ satisfying $|X_m|\le1$:
if $M_k=\sum_{m<k}X_m$ and
$V_k=\sum_{m<k}\E[X_m^2\mid\mathcal F_{\ell_m}]$, then
$\mathbb P\!\left(M_k\le-2\sqrt{V_kx}-\frac{2x}{3}\right)\le e^{-x}$
for every fixed $k$.}
to
\[
X_m
:=
Y_m^i(s_i,a_i)
-\E[Y_m^i(s_i,a_i)\mid\mathcal F_{\ell_m}],
\]
and observing that
$\sum_{m<M_t}\E[X_m^2\mid\mathcal F_{\ell_m}]\le M_t$, we obtain, for every fixed $t\le T$,
\[
\mathbb P\!\left(
\sum_{m=0}^{M_t-1}Y_m^i(s_i,a_i)
<
\frac{\delta_i}{4}M_t
-2\sqrt{M_t x}
-\frac{2x}{3}
\right)
\le e^{-x}.
\]
Using $2\sqrt{M_t x}
\le
\frac{\delta_i}{16}M_t+\frac{16x}{\delta_i}$, choosing
$x=c\Lambda_T$ with a sufficiently large universal constant $c>0$, and taking a union bound over all players, state--action pairs, and $t\le T$ gives
\[
\mathbb P\!\left(
\exists\,i,\,(s_i,a_i),\,t\le T:
\sum_{m=0}^{M_t-1}Y_m^i(s_i,a_i)
<
\frac{3\delta_i}{16}M_t
-c\frac{\Lambda_T}{\delta_i}
\right)
\le\frac{\gamma}{4}.
\]
Hence, with probability at least $1-\gamma/4$, simultaneously for all
$i$, $(s_i,a_i)$, and $t\le T$, combining this bound with
\eqref{eq:Y-to-count} gives
\begin{align}\label{eq:count-lower-proof}
N_i^t(s_i,a_i)
&\ge
\frac{3\delta_i}{16}M_t
-c\frac{\Lambda_T}{\delta_i}
-
|\mathcal S_i||\mathcal A_i|
\bigl(1+\lfloor\log_2 t\rfloor\bigr)\cr 
&\ge \frac{\delta_i}{32H_i}t
-
c\left(
\frac{\Lambda_T}{\delta_i}
+
|\mathcal S_i||\mathcal A_i|
\bigl(1+\lfloor\log_2 t\rfloor\bigr)
+1
\right),
\end{align}
where the second inequality is obtained using $M_t
=\lfloor\frac{t}{2H_i+1}\rfloor\ge
\frac{t}{3H_i}-1$ and by adjusting the universal constant. Therefore, by the definition of $K_i(T)$, for every $t\ge K_i(T)$, the number of visits to $(s_i,a_i)$ grows linearly,
\begin{align}\label{eq:linear-counter-growth}
N_i^t(s_i,a_i)\ge c\frac{\delta_i t}{H_i},
\end{align}
and consequently
\begin{align}\label{eq:beta-cover-proof}
\max_{s_i,a_i}\beta_i^t(s_i,a_i)
\le
c\sqrt{\frac{H_i\Lambda_t}{\delta_i t}}.
\end{align}

\paragraph{Step 2: Kernel approximation error.}
On the anytime confidence event $\mathcal E_{\rm conf}$ of
Lemma~\ref{lem:anytime-confidence}, $P_i\in\mathcal P_i^t$ for every
$i$ and $t$. Moreover, by construction,
\[
\hat P_i^t:=P^{\hat q_i^t}\in\mathcal P_i^t.
\]
Fix $t\ge K_i(T)$ and $(s_i,a_i)$, and let $\theta_i^t(s_i,a_i)\le t$ denote its most
recent dyadic confidence-update time. Since no new constraint for
$(s_i,a_i)$ is imposed between $\theta_i^t(s_i,a_i)$ and $t$,
\[
N_i^t(s_i,a_i)
<
2N_i^{\theta_i^t(s_i,a_i)}(s_i,a_i)
\quad \Rightarrow \quad
\beta_i^{\theta_i^t(s_i,a_i)}(s_i,a_i)
\le
\sqrt{2}\,\beta_i^t(s_i,a_i).
\]
On the anytime confidence event, both $P_i$ and
$\hat P_i^t$ belong to the confidence set $\mathcal P_i^t$ centered at the empirical kernel $\bar P_i^t$ with radius $\beta^{\theta_i^t(s_i,a_i)}(s_i,a_i)$. Thus, by the triangle inequality,
\begin{align*}
|\hat P_i^t(s_i'\mid s_i,a_i)
\!-\!P_i(s_i'\mid s_i,a_i)|
&\!\le\! |\hat P_i^t(s_i'\mid s_i,a_i)
\!-\!\bar{P}^t_i(s_i'\mid s_i,a_i)|\!+\!| \bar{P}_i^t(s_i'\mid s_i,a_i)
\!-\!P_i(s_i'\mid s_i,a_i)|\cr 
&\!\leq\! 2\beta_i^{\theta_i^t(s_i,a_i)}(s_i,a_i)
\le
2\sqrt{2}\,\beta_i^t(s_i,a_i).
\end{align*}
Therefore,
\begin{align}\label{eq:kernel-error-proof}
\max_{s_i,a_i}
\|\hat P_i^t(\cdot\mid s_i,a_i)
-P_i(\cdot\mid s_i,a_i)\|_1
\le
3|\mathcal S_i|
\max_{s_i,a_i}\beta_i^t(s_i,a_i).
\end{align}

\paragraph{Step 3: Kernel to occupancy error.} For an arbitrary time $K_i(T)\le t<T$, denote the stationary distributions of $P_i^{\pi_i^t}$ and $(\hat P_i^t)^{\pi_i^t}$ by $\nu_i^t$ and $\hat\nu_i^t$, respectively. By stationarity and Assumption~\ref{ass:ergodic},
\[
\begin{aligned}
\|\hat\nu_i^t-\nu_i^t\|_1
=
\|\hat\nu_i^t (\hat P_i^t)^{\pi_i^t}-\nu_i^t P_i^{\pi_i^t}\|_1\le
\max_{s_i}
\|(\hat P_i^t)^{\pi_i^t}(\cdot\mid s_i)-P_i^{\pi_i^t}(\cdot \mid s_i)\|_1
+
(1-\kappa)\|\hat\nu_i^t-\nu_i^t\|_1.
\end{aligned}
\]
Therefore,
\begin{align*}
\|\hat\nu_i^t-\nu_i^t\|_1
\le
\frac{1}{\kappa}
\max_{s_i}
\|(\hat P_i^t)^{\pi_i^t}(\cdot\mid s_i)-P_i^{\pi_i^t}(\cdot \mid s_i)\|_1\leq \frac{1}{\kappa}
\max_{s_i,a_i}
\|\hat P_i^t(\cdot\mid s_i,a_i)-P_i(\cdot \mid s_i,a_i)\|_1,
\end{align*}
where the last inequality follows from the convexity of the $\ell_1$-norm and because $(\hat P_i^t)^{\pi_i^t}$ and $P_i^{\pi_i^t}$ are formed using the same policy $\pi_i^t$. Because the represented and true occupancies use the same policy
$\pi_i^t$, $\|\hat\rho_i^t-\rho_i^t\|_1
=
\|\hat\nu_i^t-\nu_i^t\|_1$. Combining this identity with
\eqref{eq:kernel-error-proof} and \eqref{eq:beta-cover-proof}, for
every $K_i(T)\le t<T$,
\[
\|\hat\rho_i^t-\rho_i^t\|_1
\le \frac{3|\mathcal{S}_i|}{\kappa}\max_{s_i,a_i}\beta^t(s_i,a_i)\leq 
c\frac{|\mathcal S_i|}{\kappa}
\sqrt{\frac{H_i\Lambda_t}{\delta_i t}},
\]
which proves \eqref{eq:pointwise-occ-error}. Finally, using 
$\|\hat\rho_i^t-\rho_i^t\|_1\le2$ for $t<K_i(T)$ and
$\sum_{t=1}^Tt^{-1/2}\le2\sqrt T$,
\begin{align}\label{eq:avg-occ-error-proof}
\frac1T\sum_{t<T}\|\hat\rho_i^t-\rho_i^t\|_1
\le
c\frac{K_i(T)}{T}
+
c\frac{|\mathcal S_i|}{\kappa}
\sqrt{\frac{H_i\Lambda_T}{\delta_iT}}.
\end{align}
Since $g_i^t=g_i(q_{-i}^t)$ is independent of $s_i'$ and
$\|g_i^t\|_\infty\le1$,
\[
\begin{aligned}
\Big|\sum_{t<T}\langle\hat q_i^t-q_i^t,g_i^t\rangle\Big|
&=
\Big|\sum_{t<T}\sum_{s_i,a_i,s_i'}
\bigl(\hat q_i^t(s_i,a_i,s_i')-q_i^t(s_i,a_i,s_i')\bigr)
g_i^t(s_i,a_i)\Big|\\
&=
\Big|\sum_{t<T}\sum_{s_i,a_i}
\bigl(\hat\rho_i^t(s_i,a_i)-\rho_i^t(s_i,a_i)\bigr)
g_i^t(s_i,a_i)\Big|\le
\sum_{t<T}\|\hat\rho_i^t-\rho_i^t\|_1.
\end{aligned}
\]
Dividing by $T$ and applying \eqref{eq:avg-occ-error-proof} proves
\eqref{eq:Q-bound}. Finally, the visitation event above has probability at
least $1-\gamma/4$, while the anytime transition-confidence event
$\mathcal E_{\rm conf}$ of Lemma~\ref{lem:anytime-confidence} has
probability at least $1-\gamma/4$. Hence, by a union bound, their
intersection has probability at least $1-\gamma/2$, completing the
proof.
\end{proof}

\subsection{Proof of Lemma \ref{lem:delay-window-mixing}}

\begin{proof}
Fix $t\ge D_T$ and suppose the delay window $[t-D_T,t]$ is
\emph{clean}, meaning that no player updates its confidence set during
that window, and condition on $\mathcal F_{t-D_T}$. Freeze each player's
policy at $\pi_j^{t-D_T}$ and, as in Step 1 of the proof of
Lemma~\ref{lem:movement}, couple its actual local trajectory over
the delay window with the corresponding frozen-policy trajectory,
evolving under the true transition kernel $P_j$ and starting from
the same state at time $t-D_T$.\footnote{As in the proof of
Lemma~\ref{lem:movement}, we denote the local state trajectories of
the actual process and the frozen-policy process for player $i$ by
$\{S_i^{\tau}\}_{\tau}$ and $\{\bar{S}_i^{\tau}\}_{\tau}$,
respectively. Similarly, we denote the corresponding joint state
trajectories by $\{S^{\tau}\}_{\tau}$ and
$\{\bar{S}^{\tau}\}_{\tau}$, respectively.}
More precisely, whenever $S_j^\tau=\bar S_j^\tau$, we use a maximal
coupling of the one-step transition laws of the actual and frozen
processes. Thus, the conditional probability that the two next states
differ is bounded by the total variation distance between these
one-step transition laws. Iterating this coupling over the delay
window gives the path-separation bound in
\eqref{eq:local-path-coupling-movement}. Using this frozen-policy
coupling argument,
\begin{align}\nonumber
&\bigl\|
\mathbb P((S^t,A^t)\in\cdot\mid\mathcal F_{t-D_T})
-
\mathbb P((\bar S^t,\bar A^t)\in\cdot\mid\mathcal F_{t-D_T})
\bigr\|_{\rm TV}\cr
&\quad\le
\sum_{j=1}^n
\mathbb P\!\left(
\exists\,\tau\in[t-D_T,t]:
S_j^\tau\neq\bar S_j^\tau
\,\middle|\,
\mathcal F_{t-D_T}
\right)
+
\sup_{s}
\bigg\|
\bigotimes_{j=1}^n\pi_j^t(\cdot\mid s_j)
-
\bigotimes_{j=1}^n\pi_j^{t-D_T}(\cdot\mid s_j)
\bigg\|_{\rm TV},
\end{align}
where in the above inequality the first term bounds the probability
that the coupled state trajectories separate at some time during the
delay window. On its complement, the two processes have the same
terminal state $S^t=\bar S^t=s$. Conditional on this common state, the
players choose their actions independently according to
$\{\pi_j^t(\cdot\mid s_j)\}_{j=1}^n$ in the actual process and
$\{\pi_j^{t-D_T}(\cdot\mid s_j)\}_{j=1}^n$ in the frozen process,
which gives the second term. By bounding each summand in the above
inequality using \eqref{eq:local-path-coupling-movement} with
$D_T^2\eta L_j/\delta_j^2$ while bounding the second term using the
product-TV bound
\[
\|\bigotimes_j\mu_j-\bigotimes_j\nu_j\|_{\rm TV}
\le\sum_j\|\mu_j-\nu_j\|_{\rm TV},
\]
we can write
\begin{align}\label{eq:joint-delay-coupling}
&\bigl\|
\mathbb P((S^t,A^t)\in\cdot\mid\mathcal F_{t-D_T})
-
\mathbb P((\bar S^t,\bar A^t)\in\cdot\mid\mathcal F_{t-D_T})
\bigr\|_{\rm TV}\cr
&\qquad\le
cD_T^2\eta\sum_{j=1}^n\frac{L_j}{\delta_j^2}
+
\sup_s\sum_{j=1}^n
\|\pi_j^t(\cdot\mid s_j)
-\pi_j^{t-D_T}(\cdot\mid s_j)\|_{\rm TV}\nonumber\\
&\qquad\le
cD_T^2\eta\sum_{j=1}^n\frac{L_j}{\delta_j^2}
+
cD_T\eta\sum_{j=1}^n\frac{L_j}{\delta_j^2}
\le
cD_T^2\eta\sum_{j=1}^n\frac{L_j}{\delta_j^2},
\end{align}
where the second inequality follows from \eqref{eq:pi-movement}, and
the last inequality is obtained by adjusting the constant $c$.

Since, conditional on $\mathcal F_{t-D_T}$, the initial local states
and the frozen policies $\{\pi_j^{t-D_T}\}_{j=1}^n$ are fixed, and the
local transition and action randomizations are independent across
players, the frozen local processes evolve independently. Therefore,
their conditional joint law factorizes and we have
\begin{align}\label{eq:frozen-product-mixing}
\bigg\|
\mathbb P((\bar S^t,\bar A^t)\in\cdot\mid\mathcal F_{t-D_T})
-
\bigotimes_{j=1}^n\rho_j^{t-D_T}
\bigg\|_{\rm TV}
&=
\bigg\|
\bigotimes_{j=1}^n
\mathbb P((\bar S_j^t,\bar A_j^t)\in\cdot
\mid\mathcal F_{t-D_T})
-
\bigotimes_{j=1}^n\rho_j^{t-D_T}
\bigg\|_{\rm TV}\cr
&\le
\sum_{j=1}^n
\bigg\|
\mathbb P((\bar S_j^t,\bar A_j^t)\in\cdot
\mid\mathcal F_{t-D_T})
-
\rho_j^{t-D_T}
\bigg\|_{\rm TV}\nonumber\\
&\le
\sum_{j=1}^n
\bigg\|
\mathbb P(\bar S_j^t\in\cdot\mid\mathcal F_{t-D_T})
-
\nu_j^{t-D_T}
\bigg\|_{\rm TV}\nonumber\\
&\le
n(1-\kappa)^{D_T},
\end{align}
where the first inequality follows from the product-TV inequality, the
second follows because applying the common frozen action kernel
$\pi_j^{t-D_T}(\cdot\mid s_j)$ cannot increase total variation
distance, while the third follows from
Assumption~\ref{ass:ergodic}. Finally, combining
\eqref{eq:joint-delay-coupling} and
\eqref{eq:frozen-product-mixing} by the triangle inequality gives the
desired bound \eqref{eq:actual-product-mixing}.
\end{proof}

\subsection{Proof of Lemma \ref{lem:delay-coupling}}

\begin{proof}
We work throughout on the event of Lemma~\ref{lem:movement}, which holds with probability at least $1-\gamma/2$. Moreover, throughout we assume $t\ge D_T$ as the first $D_T$ primitive times can be bounded directly and contribute at most $cD_T/(T\delta_i)$ after normalization. Now, fix $t\ge D_T$ and call a window $[t-D_T,t]$ \emph{clean} if no player updates its confidence set during that window, and \emph{contaminated} otherwise. Since player $j$ has at most $J_j(T)$ confidence updates up to time $T$, and each such update can contaminate the delay windows associated with at most $D_T+1$ primitive times, the number of primitive times $t\in\{D_T,\ldots,T-1\}$ whose delay window $[t-D_T,t]$ is contaminated satisfies
\begin{align}\label{eq:number-contaminated-delay-windows}
\#\{\text{primitive times $t$ with contaminated delay windows}\}
\lesssim D_T\sum_{j=1}^n J_j(T).
\end{align}

For $t\ge D_T$, we use the direct decomposition
\begin{align}\label{eq:gradient-decomposition-full}
\widehat R_i^t-g_i^t
={}&
\left(
\widehat R_i^t-
\E[\widehat R_i^t\mid\mathcal F_{t-D_T}]
\right)
+
\left(
\E[\widehat R_i^t\mid\mathcal F_{t-D_T}]
-g_i^{t-D_T}
\right)
+
\left(
g_i^{t-D_T}-g_i^t
\right).
\end{align}
The first term is a delayed sampling martingale fluctuation, the second is the delayed bias, and the last term captures the drift of the stationary payoff gradient over the delay window.

\paragraph{Step 1: Sampling term.}
We now control the delayed martingale term $\widehat R_i^t-\E[\widehat R_i^t\mid\mathcal F_{t-D_T}]$.
By construction, $\|\widehat R_i^t\|_\infty\le1/\delta_i$, and hence $\big\|
\widehat R_i^t-
\E[\widehat R_i^t\mid\mathcal F_{t-D_T}]
\big\|_\infty
\le 2/\delta_i$. Moreover,
\[
\E\left[
\widehat R_i^t-
\E[\widehat R_i^t\mid\mathcal F_{t-D_T}]
\,\middle|\,
\mathcal F_{t-D_T}
\right]
=0.
\]
Because the centering is with respect to $\mathcal F_{t-D_T}$ rather than $\mathcal F_t$, partition $D_T\le t<T$ into the $D_T+1$ residue classes modulo $D_T+1$. Let $\mathcal{N}_r=\{D_T\leq t<T: (t\  \mbox{mod}\ D_T+1)= r\}$ denote the $r$th residue class. Within each residue class, successive times are separated by $D_T+1$, so for every coordinate $(s_i,a_i,s_i')$ the corresponding centered terms form a martingale-difference sequence. Moreover, since $\hat q_i^{t-D_T}$ is $\mathcal F_{t-D_T}$-measurable and has unit mass, the inner product of $\widehat R_i^t-\E[\widehat R_i^t\mid\mathcal F_{t-D_T}]$ with $\hat q_i^{t-D_T}$ also forms a martingale-difference sequence on the same residue class. Thus, Freedman's inequality gives, for every $x>0$,
\begin{align*}\nonumber
&\mathbb P\left(
\bigg|
\sum_{t\in \mathcal{N}_r}
\left(
\widehat R_i^t-
\E[\widehat R_i^t\mid\mathcal F_{t-D_T}]
\right)(s_i,a_i,s_i')
\bigg|
\gtrsim
\frac{\sqrt{|\mathcal{N}_r|x}+x}{\delta_i}
\right)
\le2e^{-x},\cr 
&\mathbb P\left(
\bigg|
\sum_{t\in \mathcal{N}_r}
\left\langle
\hat q_i^{t-D_T},
\widehat R_i^t-
\E[\widehat R_i^t\mid\mathcal F_{t-D_T}]
\right\rangle
\bigg|
\gtrsim
\frac{\sqrt{|\mathcal{N}_r|x}+x}{\delta_i}
\right)
\le2e^{-x}.
\end{align*}
Taking a union bound over all players, coordinates, and residue classes, with $x=c\Lambda_T$ for a sufficiently large universal constant $c$, gives an event $\mathcal E_{\rm grad}$ satisfying
\[
\mathbb P(\mathcal E_{\rm grad})
\ge 
1-\frac{\gamma}{4}.
\]
Since there are $D_T+1$ residue classes and $\sum_{r=0}^{D_T}|\mathcal{N}_r|\le T$, Cauchy--Schwarz gives $\sum_{r=0}^{D_T}\sqrt{|\mathcal{N}_r|}
\le \sqrt{T(D_T+1)}$.
Thus, summing the residue-class bounds, on $\mathcal E_{\rm grad}$, for any $q_i'\in\Delta_{i,\delta_i}(P_i)$, we have
\begin{align}\label{eq:gradient-martingale}
\bigg|
\sum_{D_T\leq t<T}
\left\langle
q_i'-\hat q_i^{t-D_T},
\widehat R_i^t-
\E[\widehat R_i^t\mid\mathcal F_{t-D_T}]
\right\rangle
\bigg|
\lesssim
\frac{
\sqrt{TD_T\Lambda_T}
+D_T\Lambda_T
}{\delta_i}.
\end{align}

It remains to replace $\hat q_i^{t-D_T}$ by $\hat q_i^t$ in \eqref{eq:gradient-martingale}. Using the occupancy movement bound \eqref{eq:q-q+} established in the proof of Lemma~\ref{lem:movement}, we have $\|\hat q_i^{\ell+1}-\hat q_i^{\ell}\|_1
\le \frac{L_i\eta}{\delta_i}$.
Hence, on a clean delay window,
\begin{align}\label{eq:q-clean-movement}
\|\hat q_i^t-\hat q_i^{t-D_T}\|_1
\le
\frac{D_TL_i\eta}{\delta_i}.
\end{align}
Using $\|\widehat R_i^t-\E[\widehat R_i^t\mid\mathcal F_{t-D_T}]\|_\infty\le 2/\delta_i$ and the trivial bound $\|\hat q_i^t-\hat q_i^{t-D_T}\|_1\le2$ at primitive times $t$ with contaminated delay windows, whose total number is at most of order $D_T\sum_{j=1}^nJ_j(T)$ by \eqref{eq:number-contaminated-delay-windows},
\begin{align}\label{eq:gradient-coefficient-transfer}
\bigg|\sum_{D_T\leq t<T}\left\langle \hat{q}^{t-D_T}_i-\hat{q}_i^t, \widehat R_i^t-
\E[\widehat R_i^t\mid\mathcal F_{t-D_T}]  \right\rangle\bigg|\lesssim \frac{D_TL_i\eta}{\delta_i^2}T
+
\frac{D_T}{\delta_i}
\sum_{j=1}^nJ_j(T).
\end{align}
Finally, by summing \eqref{eq:gradient-martingale} and \eqref{eq:gradient-coefficient-transfer}, and normalizing by $T$, we obtain 
\begin{align}\label{eq:final-gradient-martingale}
\bigg|
\frac{1}{T}\!\!\!\sum_{D_T\leq t<T}
\!\!\!\left\langle
q_i'-\hat q_i^{t},
\widehat R_i^t-
\E[\widehat R_i^t\mid\mathcal F_{t-D_T}]
\right\rangle
\!\bigg|
\lesssim
\frac{
\sqrt{TD_T\Lambda_T}
+D_T\Lambda_T
}{\delta_i T}+\frac{D_TL_i\eta}{\delta_i^2}
+
\frac{D_T}{\delta_iT}
\sum_{j=1}^nJ_j(T).
\end{align}

\paragraph{Step 2: Delayed bias.}
Fix an arbitrary $(s_i,a_i,s'_i)$. Since $r_i(S,A){\bf1}\{(S_i,A_i)=(s_i,a_i)\}$ is a deterministic function
of $(S,A)$ bounded by one, Lemma~\ref{lem:delay-window-mixing} implies
that, on every clean window,
\begin{align}\label{eq:gradient-unnormalized-mixing}
&\left|
\E\left[
r_i^t{\bf1}\{(S_i^t,A_i^t)\!=\!(s_i,a_i)\}
\,\middle|\,
\mathcal F_{t-D_T}
\right]
-
\rho_i^{t-D_T}(s_i,a_i)
g_i^{t-D_T}(s_i,a_i,s_i')
\right|=\cr
&\left|
\mathbb{E}\left[
r_i(S^t,A^t){\bf1}\{(S_i^t,A_i^t)\!=\!(s_i,a_i)\}
\,\middle|\,
\mathcal F_{t-D_T}
\right]
-\mathbb{E}_{(S,A)\sim\bigotimes\rho_j^{t-D_T}}
\big[
r_i(S,A){\bf1}\{(S_i,A_i)\!=\!(s_i,a_i)\}
\big]
\right|\cr 
&\qquad\lesssim
D_T^2\eta
\sum_{j=1}^n\frac{L_j}{\delta_j^2}
+
n(1-\kappa)^{D_T}.
\end{align}
Since $\hat\rho_i^{t-D_T}(s_i,a_i)$ is
$\mathcal F_{t-D_T}$-measurable, we may divide
\eqref{eq:gradient-unnormalized-mixing} by
$\hat\rho_i^{t-D_T}(s_i,a_i)$ and bring this factor inside the
conditional expectation. Using
$\hat\rho_i^{t-D_T}(s_i,a_i)\ge\delta_i$, we obtain
\begin{align}\label{eq:gradient-delayed-normalized-mixing}
&\left|
\E\bigg[
\frac{
r_i^t{\bf1}\{(S_i^t,A_i^t)=(s_i,a_i)\}
}{
\hat\rho_i^{t-D_T}(s_i,a_i)
}
\,\bigg|\,
\mathcal F_{t-D_T}
\bigg]
-
\frac{\rho_i^{t-D_T}(s_i,a_i)}
{\hat\rho_i^{t-D_T}(s_i,a_i)}
g_i^{t-D_T}(s_i,a_i,s_i')
\right|\cr 
&\qquad\qquad\lesssim
\frac{D_T^2\eta}{\delta_i}
\sum_{j=1}^n\frac{L_j}{\delta_j^2}
+
\frac{n}{\delta_i}(1-\kappa)^{D_T}.
\end{align}
We now relate the delayed normalization in
\eqref{eq:gradient-delayed-normalized-mixing} to the current
normalization used in $\widehat R_i^t$. On a clean delay window, using
\eqref{eq:q-clean-movement} and the fact that marginalization does not
increase the $\ell_1$-norm, we have
$\|\hat\rho_i^t-\hat\rho_i^{t-D_T}\|_1
\le D_TL_i\eta/\delta_i$. Since both
$\hat\rho_i^t(s_i,a_i)$ and
$\hat\rho_i^{t-D_T}(s_i,a_i)$ are at least $\delta_i$,
\[
\Big|
\frac1{\hat\rho_i^t(s_i,a_i)}
-
\frac1{\hat\rho_i^{t-D_T}(s_i,a_i)}
\Big|
\le
\frac{D_TL_i\eta}{\delta_i^3}.
\]
Consequently, for any $(s_i,a_i,s'_i)$ we have
\begin{align}\label{eq:gradient-denominator-bias}
&\left|
\E[\widehat R_i^t(s_i,a_i,s_i')\mid\mathcal F_{t-D_T}]
-
\E\bigg[
\frac{
r_i^t{\bf1}\{(S_i^t,A_i^t)=(s_i,a_i)\}
}{
\hat\rho_i^{t-D_T}(s_i,a_i)
}
\,\bigg|\,
\mathcal F_{t-D_T}
\bigg]
\right|
\le
\frac{D_TL_i\eta}{\delta_i^3}.
\end{align}
Finally, since $\|g_i^{t-D_T}\|_\infty\le1$ and $\hat\rho_i^{t-D_T}(s_i,a_i)\ge\delta_i$,
\begin{align}\label{eq:G-exact}
\bigg\|
\frac{\rho_i^{t-D_T}}
{\hat\rho_i^{t-D_T}}
g_i^{t-D_T}
-
g_i^{t-D_T}
\bigg\|_\infty
&\le
\frac{
\|\rho_i^{t-D_T}-\hat\rho_i^{t-D_T}\|_1
}{\delta_i}\lesssim
\frac{|\mathcal S_i|}{\kappa\delta_i}
\sqrt{
\frac{H_i\Lambda_T}
{\delta_i(t-D_T)}
}.
\end{align}
where the second inequality follows by the pointwise occupancy bound
\eqref{eq:pointwise-occ-error} of Lemma~\ref{lem:movement}. Combining \eqref{eq:gradient-delayed-normalized-mixing},
\eqref{eq:gradient-denominator-bias}, and \eqref{eq:G-exact}, for
every $t$ with a clean delay window and $t-D_T\ge K_i(T)$, we obtain
\begin{align}\nonumber
\left\|
\E[\widehat R_i^t\mid\mathcal F_{t-D_T}]
-
g_i^{t-D_T}
\right\|_\infty
\lesssim{}
\frac{D_TL_i\eta}{\delta_i^3}
+
\frac{D_T^2\eta}{\delta_i}
\sum_{j=1}^n\frac{L_j}{\delta_j^2}
+
\frac{n}{\delta_i}(1-\kappa)^{D_T}+
\frac{|\mathcal S_i|}{\kappa\delta_i}
\sqrt{
\frac{H_i\Lambda_T}
{\delta_i(t-D_T)}
}.
\end{align}

For the remaining times, we use the trivial $c/\delta_i$ bound.
In particular, by \eqref{eq:number-contaminated-delay-windows}, there
are at most of order $D_T\sum_{j=1}^nJ_j(T)$ times whose associated delay
windows are contaminated, and there are at most $K_i(T)$ times in
the range $D_T\le t<T$ for which $t-D_T<K_i(T)$. Therefore, using
$\sum_{u=1}^T u^{-1/2}\le2\sqrt T$, we obtain
\begin{align}\label{eq:gradient-delayed-bias-sum}
\frac1T
\sum_{D_T\le t<T}
\left\|
\E[\widehat R_i^t\mid\mathcal F_{t-D_T}]
-g_i^{t-D_T}
\right\|_\infty
&\lesssim{}
\frac{D_TL_i\eta}{\delta_i^3}
+
\frac{D_T^2\eta}{\delta_i}
\sum_{j=1}^n\frac{L_j}{\delta_j^2}
+
\frac{n}{\delta_i}(1-\kappa)^{D_T}
\nonumber\\
&+
\frac{|\mathcal S_i|}{\kappa\delta_i}
\sqrt{
\frac{H_i\Lambda_T}{\delta_iT}
}
+
\frac{D_T}{T\delta_i}
\sum_{j=1}^nJ_j(T)
+
\frac{K_i(T)}{T\delta_i}.
\end{align}

\paragraph{Step 3: Target drift.}
We now establish a bound on the target drift
$\|g_i^t-g_i^{t-D_T}\|_{\infty}$.
Since no confidence update occurs during a clean delay window,
\eqref{eq:kernel-movement} gives, for every
$\ell\in[t-D_T,t-1]$,
\[
\max_{s_j}
\|P_j^{\pi_j^{\ell+1}}(\cdot\mid s_j)
-P_j^{\pi_j^{\ell}}(\cdot\mid s_j)\|_1
\le
\frac{2L_j\eta}{\delta_j^2}.
\]
By stationarity and Assumption~\ref{ass:ergodic},
\[
\begin{aligned}
\|\nu_j^{\ell+1}-\nu_j^{\ell}\|_1
\!=\!
\|\nu_j^{\ell+1}P_j^{\pi_j^{\ell+1}}
\!\!\!\!\!-\nu_j^{\ell}P_j^{\pi_j^{\ell}}\|_1\!\le\!
\max_{s_j}
\|P_j^{\pi_j^{\ell+1}}\!\!\!\!(\cdot\mid s_j)
-P_j^{\pi_j^{\ell}}(\cdot\mid s_j)\|_1
\!+\!(1\!-\!\kappa)\|\nu_j^{\ell+1}-\nu_j^{\ell}\|_1,
\end{aligned}
\]
and hence
\[
\|\nu_j^{\ell+1}-\nu_j^{\ell}\|_1
\le
\frac{2L_j\eta}{\kappa\delta_j^2}.
\]
Since
$\rho_j^{\ell}(s_j,a_j)=\nu_j^{\ell}(s_j)\pi_j^{\ell}(a_j\mid s_j)$,
using \eqref{eq:pi-movement} gives
\begin{align}\label{eq:occupancy-bound-difference}
\|\rho_j^{\ell+1}-\rho_j^{\ell}\|_1
&\le
\|\nu_j^{\ell+1}-\nu_j^{\ell}\|_1
+
\max_{s_j}
\|\pi_j^{\ell+1}(\cdot\mid s_j)-\pi_j^{\ell}(\cdot\mid s_j)\|_1\lesssim
\frac{L_j\eta}{\kappa\delta_j^2}.
\end{align}
Therefore, telescoping over the clean delay window and using the
multilinearity of $g_i$ in the opponents' state--action occupancies,
together with the fact that rewards lie in $[0,1]$, gives
\begin{align}\label{eq:target-drift-window}
\|g_i^t-g_i^{t-D_T}\|_\infty
&\le
\sum_{j\ne i}\|\rho_j^t-\rho_j^{t-D_T}\|_1
\le
\sum_{j\ne i}\sum_{\ell=t-D_T}^{t-1}
\|\rho_j^{\ell+1}-\rho_j^{\ell}\|_1\lesssim
D_T\eta
\sum_{j\ne i}\frac{L_j}{\kappa\delta_j^2}.
\end{align}
Here, the first inequality follows from the multilinearity of $g_i$ in
the opponents' state--action occupancies and the fact that rewards lie
in $[0,1]$, the second inequality is obtained by telescoping over the
clean delay window, while the last follows from
\eqref{eq:occupancy-bound-difference}. Now, using the bound \eqref{eq:occupancy-bound-difference} for clean windows and the trivial bound $\|g_i^{t-D_T}-g_i^t\|_\infty\le1$ on contaminated windows, we obtain
\begin{align}\label{eq:gradient-delay-sum}
\frac1T
\sum_{D_T\le t<T}
\|g_i^{t-D_T}-g_i^t\|_\infty
\lesssim
D_T\eta
\sum_{j\ne i}\frac{L_j}{\kappa\delta_j^2}
+
\frac{D_T}{T}
\sum_{j=1}^nJ_j(T).
\end{align}

\noindent 
{\bf Putting all together:} Since $\|q_i'-\hat q_i^t\|_1\le2$, combining
\eqref{eq:gradient-decomposition-full},
\eqref{eq:final-gradient-martingale},
\eqref{eq:gradient-delayed-bias-sum}, and 
\eqref{eq:gradient-delay-sum}, and absorbing the first $D_T$
primitive times and the smaller contaminated-window terms, gives
\begin{align*}
\frac1T\left|
\sum_{t<T}
\langle q_i'-\hat q_i^t,\widehat R_i^t-g_i^t\rangle
\right|
\lesssim{}&
\frac{
\sqrt{TD_T\Lambda_T}
+D_T\Lambda_T
}{T\delta_i}
+\frac{D_T\eta L_i}{\delta_i^3}+
\frac{D_T^2\eta}{\delta_i}
\sum_{j=1}^n\frac{L_j}{\delta_j^2}
+
\frac{|\mathcal S_i|}{\kappa\delta_i}
\sqrt{
\frac{H_i\Lambda_T}{\delta_iT}
}\\
&+
D_T\eta
\sum_{j\ne i}\frac{L_j}{\kappa\delta_j^2}
+
\frac{D_T}{T\delta_i}
\sum_{j=1}^nJ_j(T)
+
\frac{D_T+K_i(T)}{T\delta_i},
\end{align*}
where since $(1-\kappa)^{D_T}
\le
e^{-\kappa D_T}
\le
\frac{\gamma}{8nT^3}$, the exponentially small mixing term in
\eqref{eq:gradient-delayed-bias-sum} is absorbed into the displayed
bound. Since $\mathbb P(\mathcal E_{\rm grad})\ge1-\gamma/4$, the
result holds on the event of Lemma~\ref{lem:movement} and an
additional event of probability at least $1-\gamma/4$, uniformly over
$q_i'\in\Delta_{i,\delta_i}(P_i)$. This proves
\eqref{eq:cumulative-gradient-error}.
\end{proof}

\section{Omitted Proofs for Section \ref{sec:asymptotic}}

\subsection{Proof of Lemma \ref{lem:asymptotic-visitation}}

\begin{proof}
Partition primitive time into deterministic blocks
$B_m=\{\ell_m,\ldots,\ell_m+2H_i\}$ with
$\ell_m=m(2H_i+1)$, as in the proof of Lemma~\ref{lem:movement}.
Because $\eta_t\to0$, for every sufficiently late block the
small-movement condition~\eqref{eq:T-cover-threshold}, with $\eta$
replaced by the stepsize at the beginning of the block, is satisfied.
Thus, conditional on the filtration at the deterministic left endpoint,
the event that the prescribed pair $(s_i,a_i)$ is observed in the
block, or that a confidence update occurs in the block, has probability
at least $\delta_i/4$. We may therefore apply the same Freedman and
union-bound argument leading to \eqref{eq:linear-counter-growth}.
To extend this finite-horizon bound to all sufficiently large times,
we apply it on the nested dyadic horizons $[0,T_k]$, $T_k=2^k$, and use
Borel--Cantelli to ensure that all sufficiently large such horizons are
simultaneously good. More precisely, apply the argument on horizon $T_k$
with failure probability $\gamma_k=c_0/(k+1)^2$, where $c_0>0$ is
fixed. Since $\sum_{k\ge0}\gamma_k<\infty$, the Borel--Cantelli lemma
implies that, almost surely, there exists a finite random $k_0$ such
that the corresponding finite-horizon event holds for every $k\ge k_0$,
simultaneously over the finitely many players and local pairs.
Consequently, almost surely, for all sufficiently large $k$, the bound
\eqref{eq:linear-counter-growth} holds simultaneously for all
$K_i(T_k)\le t\le T_k$. For every sufficiently large $t$, let
$k=\lceil\log_2 t\rceil$. Then $t\le T_k<2t$ and
$K_i(T_k)=O(\log(t+2))\le t$.
Hence \eqref{eq:linear-counter-growth} holds for all sufficiently large
$t$, which gives \eqref{eq:asymptotic-count}. Thus, almost surely, we
may choose a single finite random time $T_0$ large enough that
$k=\lceil\log_2 t\rceil\ge k_0$ and $K_i(T_k)\le t$ for every player
$i$ and every $t\ge T_0$, as claimed.

By substituting
\eqref{eq:asymptotic-count} into the definition of the confidence
radius $\beta_i^t(s_i,a_i)$ we get \eqref{eq:asymptotic-beta}. Moreover, 
\[
q_i^t(s_i,a_i,s_i')
=
\rho_i^t(s_i,a_i)P_i(s_i'\mid s_i,a_i),
\qquad
\hat q_i^t(s_i,a_i,s_i')
=
\hat\rho_i^t(s_i,a_i)
\hat P_i^t(s_i'\mid s_i,a_i).
\]
Therefore,
\begin{align}
\|\hat q_i^t-q_i^t\|_1
&\le
\sum_{s_i,a_i}
|\hat\rho_i^t(s_i,a_i)-\rho_i^t(s_i,a_i)|
\sum_{s_i'}\hat P_i^t(s_i'\mid s_i,a_i)
\notag\\
&\quad+
\sum_{s_i,a_i}
\rho_i^t(s_i,a_i)
\sum_{s_i'}
|\hat P_i^t(s_i'\mid s_i,a_i)
-P_i(s_i'\mid s_i,a_i)|
\notag\\
&\le
\|\hat\rho_i^t-\rho_i^t\|_1
+
\max_{s_i,a_i}
\|\hat P_i^t(\cdot\mid s_i,a_i)
-P_i(\cdot\mid s_i,a_i)\|_1\cr 
&\leq c\frac{|\mathcal{S}_i|}{\kappa}\max_{s_i,a_i}\beta_i^t(s_i,a_i)
+3|\mathcal{S}_i|\max_{s_i,a_i}\beta_i^t(s_i,a_i)
=O\bigg(\!
\sqrt{\frac{\log(t+2)}{t+1}}
\bigg)\to 0,
\label{eq:qhat-q-decomposition}
\end{align}
where the third inequality uses \eqref{eq:pointwise-occ-error} and \eqref{eq:kernel-error-proof}, and the last equality is obtained by \eqref{eq:asymptotic-beta}.   
\end{proof}

\vspace{-0.7cm}

\subsection{Proof of Lemma \ref{lem:restart-jumps}}

\begin{proof}
Fix $x\in\Delta_{i,\delta_i}(\mathcal P_i^t)$ and let
$P^x$ denote its induced transition kernel. By definition,
$P^x\in\mathcal P_i^t$. On the anytime confidence event $\mathcal E_{\rm conf}$,
$P_i\in\mathcal P_i^t$ as well. Therefore, the same dyadic-confidence
argument used to establish \eqref{eq:kernel-error-proof}, with
$\hat P_i^t$ replaced by $P^x$, gives
\begin{align}\label{eq:kernel-violation-polytope}
\max_{s_i,a_i}
\|P^x(\cdot\mid s_i,a_i)-P_i(\cdot\mid s_i,a_i)\|_1
\le
3|\mathcal S_i|
\max_{s_i,a_i}\beta_i^t(s_i,a_i).
\end{align}
Now $x$ satisfies the normalization, flow, nonnegativity, and
shrinkage constraints defining $\Delta_{i,\delta_i}(P_i)$. Its only
possible violation of this polytope is therefore in the true-kernel
constraints
\[
x(s_i,a_i,s_i')
=
\rho_x(s_i,a_i)P_i(s_i'\mid s_i,a_i).
\]
Since
$x(s_i,a_i,s_i')
=\rho_x(s_i,a_i)P^x(s_i'\mid s_i,a_i)$ and
$\rho_x(s_i,a_i)\le1$, the kernel bound \eqref{eq:kernel-violation-polytope} controls the
violation of these linear constraints. Hoffman's error bound\footnote{Hoffman's error bound states that, for a
nonempty polyhedron defined by a fixed finite system of linear equalities
and inequalities, the distance of any point to the polyhedron is bounded
by a constant times the magnitude of its constraint violations. The
constant depends only on the coefficients of the fixed constraint
system.} for the
fixed nonempty polytope $\Delta_{i,\delta_i}(P_i)$ consequently gives $\operatorname{dist}_1\!\left(
x,\Delta_{i,\delta_i}(P_i)\right)
\le O\big(|\mathcal S_i|
\max_{s_i,a_i}\beta_i^t(s_i,a_i)\big).$ Taking the supremum over
$x\in\Delta_{i,\delta_i}(\mathcal P_i^t)$ and applying
\eqref{eq:asymptotic-beta} proves \eqref{eq:set-convergence}.

We next control the change of the iterate at a confidence update 
$t\in\mathcal R_i$. Since the confidence sets are nested and contain
the true kernel $P_i$,  $\Delta_{i,\delta_i}(P_i)
\subseteq
\Delta_{i,\delta_i}(\mathcal P_i^{t+1})
\subseteq
\Delta_{i,\delta_i}(\mathcal P_i^t)$. Thus,
\begin{align}\label{eq:dis-beta-o}
\operatorname{dist}_1\!\left(
\hat q_i^t,\Delta_{i,\delta_i}(\mathcal P_i^{t+1})\right)
\le
\operatorname{dist}_1\!\left(
\hat q_i^t,\Delta_{i,\delta_i}(P_i)\right)
\le
O\big(|\mathcal S_i|
\max_{s_i,a_i}\beta_i^t(s_i,a_i)\big),
\end{align}
where the second inequality follows from \eqref{eq:set-convergence} since $\hat q_i^t\in\Delta_{i,\delta_i}(\mathcal P_i^t)$. At a confidence update time $t$, using the Euclidean proximal update rule in the algorithm, we can write
\begin{align}\nonumber
\|\hat q_i^{t+1}-\hat q_i^t\|_2
&\!=\!
\left\|
\operatorname{proj}_{\Delta_{i,\delta_i}(\mathcal P_i^{t+1})}
\!\left(\hat q_i^t\!+\!\eta_t\widehat R_i^t\right)
-\hat q_i^t
\right\|_2 \cr
&\!\le\!
\left\|
\operatorname{proj}_{\Delta_{i,\delta_i}(\mathcal P_i^{t+1})}
\!\left(\hat q_i^t\!+\!\eta_t\widehat R_i^t\right)
-\operatorname{proj}_{\Delta_{i,\delta_i}(\mathcal P_i^{t+1})}
\!\left(\hat q_i^t\right)
\right\|_2\!\!+
\left\|
\operatorname{proj}_{\Delta_{i,\delta_i}(\mathcal P_i^{t+1})}
\!\left(\hat q_i^t\right)
-\hat q_i^t
\right\|_2 \cr
&\!\le
\eta_t\|\widehat R_i^t\|_2
+
\operatorname{dist}_2\!\left(
\hat q_i^t,\Delta_{i,\delta_i}(\mathcal P_i^{t+1})\right)\cr 
&\!\leq\! \eta_t \frac{\sqrt{|\mathcal S_i|}}{\delta_i}+\operatorname{dist}_1\!\left(
\hat q_i^t,\Delta_{i,\delta_i}(\mathcal P_i^{t+1})\right)\cr
&\!\leq\! \eta_t \frac{\sqrt{|\mathcal S_i|}}{\delta_i}+O\big(|\mathcal{S}_i|\max_{s_i,a_i}\beta_i^t(s_i,a_i)\big),
\end{align}
where the second inequality uses
nonexpansiveness of projection, the third inequality holds because $\operatorname{dist}_2\le\operatorname{dist}_1$ and 
$\|\widehat R_i^t\|_2\le\sqrt{|\mathcal S_i|}/\delta_i$, and the last inequality follows from \eqref{eq:dis-beta-o}. Therefore, using the equivalence of norms in the fixed finite-dimensional occupancy space,
\[
\|\hat q_i^{t+1}-\hat q_i^t\|_1
\le
\theta_i\left(
\max_{s_i,a_i}\beta_i^t(s_i,a_i)+\eta_t
\right),
\]
where $\theta_i<\infty$ depends only on the fixed problem parameters such as $\delta_i$ and $|\mathcal{S}_i|$.

It remains to sum this bound over the confidence-update times $t\in \mathcal{R}_i$.
Fix a local pair $(s_i,a_i)$ and let $t_k$ be the time at which its
counter reaches the $k$th dyadic threshold, i.e., $N^{t_k}_i(s_i,a_i)=2^k$. Since the counter
can increase by at most one at each primitive time, $t_k\ge2^k$. Moreover, \eqref{eq:asymptotic-count} implies that, for all sufficiently large $k$, $t_k\leq c_i2^k$ for some constant $c_i<\infty$, and hence $t_k\asymp2^k$.
Thus, by \eqref{eq:asymptotic-beta} and
$\eta_t=\eta_0(t+1)^{-a}$, we obtain $\max_{s_i,a_i}\beta_i^{t_k}(s_i,a_i)
=O(\sqrt{k}\,2^{-k/2})$ and $\eta_{t_k}=O(2^{-ak})$, where we note that both these sequences are summable. Since player $i$ has only finitely many
local state--action pairs $(s_i,a_i)$, summing over their dyadic confidence
updates gives
\[
\sum_{t\in\mathcal R_i}
\|\hat q_i^{t+1}-\hat q_i^t\|_1\leq \theta_i |\mathcal{S}_i||\mathcal{A}_i|O\Big(\sum_k\big(\sqrt{k}\,2^{-k/2}+2^{-ak}\big)\Big)<\infty,
\]
which proves \eqref{eq:restart-jump-summable}. 
\end{proof}

\end{document}